\documentclass[11pt]{article}

\usepackage[final]{acl}

\usepackage{times}
\usepackage{latexsym}

\usepackage[T1]{fontenc}

\usepackage[utf8]{inputenc}

\usepackage{microtype}

\usepackage{inconsolata}

\usepackage{graphicx}

\usepackage{amsmath}
\usepackage{amssymb}
\usepackage{mathtools}
\usepackage{amsthm}
\usepackage{bbm}
\usepackage{enumitem}
\usepackage{multirow}
\usepackage{wrapfig}
\usepackage{float}
\usepackage{hyperref}
\usepackage{hhline}
\usepackage{booktabs}
\usepackage[table]{xcolor}
\usepackage{makecell}
\usepackage{subcaption}
\usepackage{pdflscape}
\usepackage{adjustbox}

\usepackage[capitalize,noabbrev]{cleveref}
\usepackage{xcolor}
\usepackage[most]{tcolorbox}

\usepackage{tikz}
\usetikzlibrary{decorations.pathmorphing}
\newcommand{\utilde}[1]{#1\rlap{\hspace{-0.25mm}\textsuperscript{\scalebox{0.88}{$\dagger$}}}}

\newcommand{\pairmethod}[2]{\shortstack[c]{#1/\\#2}}

\definecolor{boxframe}{RGB}{110,110,110}
\definecolor{boxhead}{RGB}{50,50,50}

\theoremstyle{plain}
\newtheorem{theorem}{Theorem}[section]
\newtheorem{proposition}[theorem]{Proposition}
\newtheorem{lemma}[theorem]{Lemma}
\newtheorem{corollary}[theorem]{Corollary}
\theoremstyle{definition}

\theoremstyle{remark}

\usepackage[textsize=tiny]{todonotes}

\title{F-GRPO: Don't Let Your Policy Learn the Obvious and Forget the Rare}

\author{
  \textbf{Daniil Plyusov}\textsuperscript{1,2,*},
  \textbf{Alexey Gorbatovski}\textsuperscript{1,*},
  \textbf{Boris Shaposhnikov}\textsuperscript{1},
  \textbf{Viacheslav Sinii}\textsuperscript{1},
  \\
  \textbf{Alexey Malakhov}\textsuperscript{1},
  \textbf{Daria Korotyshova}\textsuperscript{1},
  \textbf{Daniil Gavrilov}\textsuperscript{1}
  \\
  \textsuperscript{1}T-Tech,
  \textsuperscript{2}Saint Petersburg Electrotechnical University ``LETI''
  \\
  \small{\textsuperscript{*}Equal contribution.}
  \\
  \small{\textbf{Correspondence:}
  \href{mailto:a.gorbatovskiy@t-tech.dev}{a.gorbatovskiy@t-tech.dev}}
}

\begin{document}
\maketitle
\begin{abstract}
Reinforcement Learning with Verifiable Rewards (RLVR) is commonly based on group sampling to estimate advantages and stabilize policy updates. In practice, computational limits often rule out very large groups, so training proceeds with finite rollout sets that can reinforce only the correct behavior they expose. At practical group sizes, updates can miss rare-correct trajectories while still containing mixed rewards, concentrating probability on more common sampled solutions. We derive the probability of such prompt-local tail-miss events as a function of group size, showing non-monotonic behavior, and in the categorical abstraction characterize how unsampled-correct mass can shrink even as total correct mass grows. Motivated by this analysis, we propose a difficulty-aware scaling coefficient, inspired by Focal loss, that down-weights updates on high-success sampled groups. Empirically, categorical simulation illustrates the same effect in the categorical setting, Maze provides a single-solution test, and LLM experiments include a representative GRPO group-size sweep together with fixed-$N$ transfer across GRPO, DAPO, and CISPO. On Qwen2.5-7B at $N{=}8$, our method improves average math pass@256 from 64.1 → 70.3 (GRPO), 69.3 → 72.5 (DAPO), and 73.2 → 76.8 (CISPO); OOD pass@256 also improves in all three cases, without increasing group size or computational cost.
\end{abstract}

\section{Introduction}
\label{sec:intro}

Reinforcement Learning with Verifiable Rewards (RLVR) has become a standard paradigm for post-training large language models (LLMs), enabling strong gains on reasoning-intensive tasks without reliance on human preference data \citep{rlvr_survey}. By leveraging automatically checkable reward signals, RLVR has driven state-of-the-art performance in mathematical reasoning \citep{aime_amc}, code generation \citep{swe-bench}, and general problem solving \citep{arc-agi2}, and is now widely adopted in large-scale post-training \citep{deepseek_r1, qwen3, kimi-k2, grpo}.

Despite these successes, a growing body of work suggests that RLVR does not primarily introduce new knowledge, but instead sharpens the output distribution toward solutions already accessible to the base model \citep{rl_passk_incentivize, can_grpo_transcend, wu2025invisible, dang2025weight}. Empirical evidence based on pass@$k$ \citep{pass_at_k} indicates that RLVR-trained models may underperform their base counterparts at sufficiently large sampling budgets, consistent with a narrowing of solution diversity \citep{rl_squeezes}. At the same time, other studies argue that prolonged or carefully scaled RL can expand the effective reasoning boundary \citep{prorl, rl_new_skills}, leaving the role of RLVR an open question.

\begin{figure*}[t]
\centering
\includegraphics[width=\textwidth]{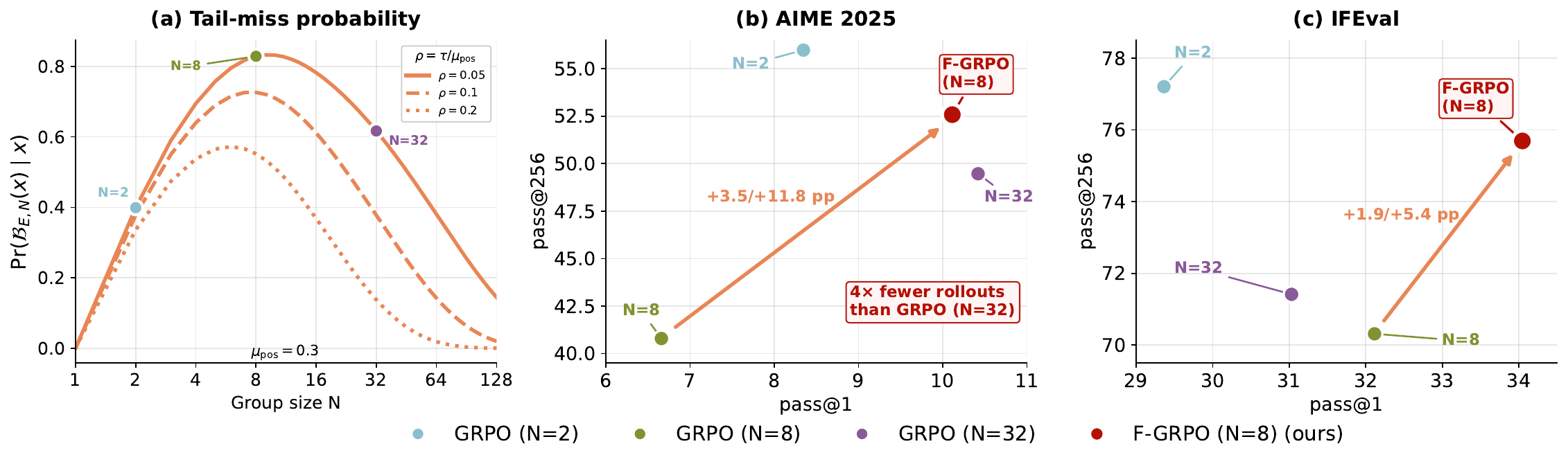}
\vspace{-1.7em}
\caption{
(a) Schematic theoretical tail-miss probability: an update is \emph{active} (mixed rewards) but \emph{misses} rare-correct solutions within a prompt. The curve is non-monotonic in group size \(N\): small groups are often inactive, large groups cover rare modes, and intermediate groups can combine activity with incomplete coverage.
(b,c) Empirical consequences on AIME 2025 (math) and IFEval (OOD): GRPO at $N{=}8$ improves pass@1 over $N{=}2$ but degrades pass@256, consistent with the sharpening regime. F-GRPO at $N{=}8$ achieves pass@256 comparable to GRPO at $N{=}32$ while maintaining pass@1, with $4{\times}$ fewer rollouts per prompt.
}
\label{fig:main}
\vspace{-1.0em}
\end{figure*}

Most modern RLVR systems rely on group-relative methods such as GRPO \citep{grpo} and its variants \citep{dapo, cispo, drgrpo}, which compute advantages from multiple rollouts per prompt. The group size thus becomes a critical design choice, yet existing work provides conflicting guidance: \citet{it_takes_2} show that two rollouts suffice and connect GRPO to DPO \citep{dpo}, while \citet{brorl} advocate scaling rollouts to broaden exploration. Since group size controls how much of a prompt's rollout space is exposed to each update, understanding its interaction with sharpening is essential. This raises a fundamental question: \emph{how does group size affect the optimization dynamics of group-relative RLVR with binary rewards, and can we mitigate sharpening without scaling computational cost?}

Our contributions are as follows: 
\begin{itemize}[leftmargin=*, itemsep=2pt, topsep=4pt]
\item We make explicit the tail-miss probability, characterizing when an active RLVR update omits a rare correct subset for a prompt. Its non-monotonic dependence on group size helps reconcile prior conflicting findings: small groups can preserve high pass@$k$ through inactivity, large groups through coverage, while intermediate groups, common under compute constraints, may maximize active tail-miss events.
\item Building on the categorical framework of \citet{brorl}, we analyze redistribution within the correct set and show that unsampled-correct mass can decrease even when total correct mass increases.
\item We use a single-solution Maze setup, showing that pass@$k$ can degrade even when each prompt has a unique correct sequence.
\item We propose F-GRPO, a difficulty-aware advantage scaling applicable to group-relative objectives including GRPO, DAPO, and CISPO, and show at fixed $N{=}8$ that it transfers across methods and models, consistently improving math and OOD pass@256 while preserving or improving pass@1, without additional rollout cost.
\end{itemize}

\section{Preliminaries}
\label{sec:prelim}

\subsection{Reinforcement Learning with Verifiable Rewards}
\label{sec:prelim:rlvr}

We consider RLVR for language model reasoning. Given a prompt $x$ from a distribution $D$, the policy $\pi_\theta$ generates complete responses (trajectories). We sample a group of $N$ i.i.d.\ rollouts $\{o_i\}_{i=1}^N \sim \pi_\theta(\cdot \mid x)$ and assign binary outcome rewards

\noindent
\begin{equation}
R_i = R_w + (R_c - R_w)\,\mathbb{I}[o_i \text{ is correct}]
\label{eq:binary_reward}
\end{equation}
where $R_c > R_w$ (typically $R_c = 1$, $R_w \in \{0, -1\}$). We work with outcome-level rewards: the reward depends only on final correctness.

For each prompt $x$, let $\Omega_x$ denote the space of complete rollouts and $\mathcal{C}(x)\subseteq \Omega_x$ the subset of correct rollouts. The per-prompt success probability is

\noindent
\begin{equation}
\mu_{\mathrm{pos}}(x) := \Pr_{o\sim\pi_\theta(\cdot|x)}[o\in \mathcal{C}(x)].
\label{eq:mu_pos_def}
\end{equation}
In training, prompts are drawn from the empirical prompt distribution; the finite-sampling analysis below is stated conditionally on a fixed prompt $x$.

\subsection{Group-Relative Policy Optimization}
\label{sec:prelim:grpo}

Group Relative Policy Optimization (GRPO)~\citep{grpo} eliminates the learned value function by computing advantages relative to the sampled group. For a prompt $x$ with $N$ rollouts $\{o_i\}_{i=1}^N$ and rewards $\{R_i\}_{i=1}^N$, the group-relative advantage is

\noindent
\begin{equation}
\widehat{A}_i^{\mathrm{GRPO}} = \frac{R_i - \bar{R}}{\sigma_R + \epsilon},
\label{eq:grpo_adv}
\end{equation}
where $\bar{R} = \frac{1}{N}\sum_{j=1}^N R_j$ and $\sigma_R = \mathrm{std}(\{R_j\}_{j=1}^N)$.

GRPO optimizes a clipped surrogate objective using token-level importance ratios; Appendix~\ref{app:objective_details} records the GRPO, DAPO, and CISPO objective details used in our experiments.

A key property of group-relative advantages is that when all sampled rewards are identical ($\sigma_R = 0$), we have $\widehat{A}_i^{\mathrm{GRPO}} = 0$ for all $i$, which yields zero learning signal. This occurs when all rollouts are correct or all are incorrect.

\subsection{Categorical Policy Framework}
\label{sec:prelim:categorical}

To analyze how RLVR updates redistribute probability mass, we adopt the categorical policy framework of \citet{brorl}. Consider $p = \mathrm{softmax}(z)$ over a finite action space $\mathcal{A}$, partitioned into correct actions $\mathcal{P}$ and incorrect actions $\mathcal{I} = \mathcal{A} \setminus \mathcal{P}$. Define the total correct and incorrect masses

\noindent
\begin{equation}
Q_{\mathrm{pos}} := \sum_{i \in \mathcal{P}} p_i, \quad Q_{\mathrm{neg}} := 1 - Q_{\mathrm{pos}}.
\label{eq:Q_masses}
\end{equation}
Draw $N$ i.i.d.\ samples from $p$. Let $S_+ \subseteq \mathcal{P}$ and $S_- \subseteq \mathcal{I}$ denote sampled correct and incorrect actions, and let $U = \mathcal{A} \setminus (S_+ \cup S_-)$ denote the unsampled actions. Define $P_{\mathrm{pos}}:=\sum_{i\in S_+}p_i$, $P_{\mathrm{neg}}:=\sum_{i\in S_-}p_i$, $S_{+,2}:=\sum_{i\in S_+}p_i^2$, $S_{-,2}:=\sum_{i\in S_-}p_i^2$, $U_{\mathrm{pos},2} := \sum_{i \in U \cap \mathcal{P}} p_i^2$, and $U_{\mathrm{neg},2} := \sum_{i \in U \cap \mathcal{I}} p_i^2$. Assign rewards as in~\eqref{eq:binary_reward} for sampled actions, with $R_i = 0$ for unsampled. The batch baseline is $S_R := R_c P_{\mathrm{pos}} + R_w P_{\mathrm{neg}}$.

We analyze TRPO-style linear surrogate updates and their unbiased Monte Carlo estimates. Under standard regularity conditions, expectation and differentiation may be interchanged \citep{asmussen2007stochastic, brorl}. Differentiating the sample surrogate with respect to the logits \(z_j\) (using \(\partial p_i/\partial z_j = p_i(\delta_{ij}-p_j)\)) yields the one-step logit update

\noindent
\begin{equation}
\Delta z_i = \frac{\eta}{N} p_i (R_i - S_R),
\label{eq:dz_update}
\end{equation}
where $\eta$ is the learning rate. For unsampled actions ($i \in U$), this reduces to $\Delta z_i = -\frac{\eta}{N} S_R p_i$.

From this update rule, \citet{brorl} derive the one-step change in total correct mass; Appendix~\ref{app:categorical_prelim} records the expression and term interpretation.

This categorical framework captures one-step local redistribution under sampled updates and provides the controlled abstraction used in Section~\ref{sec:theory:unsampled}. To avoid conflation, we keep its notation separate from trajectory-level quantities: $\mu_{\mathrm{pos}}(x)$ denotes per-prompt success probability~\eqref{eq:mu_pos_def}, while $Q_{\mathrm{pos}}$ denotes positive mass in the categorical abstraction~\eqref{eq:Q_masses}.

\section{Finite-Sampling Bias in Group-Relative RLVR}
\label{sec:theory}

Recent work offers seemingly conflicting guidance on group size in RLVR: very small groups ($N=2$) can match larger ones efficiently~\citep{it_takes_2}, moderate sizes improve pass@1 while sharpening the distribution~\citep{rewarding_unlikely}, and large groups stabilize learning~\citep{brorl}. These recommendations can be understood by analyzing RLVR at the level where group-relative sampling acts directly: a finite set of rollouts for a prompt. We study two complementary consequences of this finite sampling. First, for a fixed prompt, an update may be active while omitting a low-mass correct region. Second, given the sampled set, local redistribution can move probability away from unsampled correct outcomes. Together, these quantities describe how finite groups can improve accuracy while narrowing coverage.

\subsection{Tail-miss probability and the group size trade-off}
\label{sec:theory:tailmiss}

We first isolate a single prompt $x$ and ask when a finite group can update the policy without ever sampling a low-mass correct region. Fix a target subset $E_x\subseteq\mathcal{C}(x)$ of correct rollouts. Let

\noindent
\begin{equation}
\tau_E(x) := \Pr_{o\sim\pi_\theta(\cdot|x)}[o\in E_x],
\label{eq:tau_def}
\end{equation}
and assume $0 < \tau_E(x) < \mu_{\mathrm{pos}}(x)$. We call $E_x$ rare under the current policy when $\tau_E(x)/\mu_{\mathrm{pos}}(x)$ is small.

Let $X_x$ denote the number of correct rollouts among the $N$ samples for prompt $x$. For group-relative methods such as GRPO, the learning signal vanishes when all sampled rewards are identical, i.e., $X_x \in \{0, N\}$. Define the \emph{active event}

\noindent
\begin{equation}
\mathcal{A}_N(x) := \{0 < X_x < N\},
\label{eq:active_event}
\end{equation}
with probability $\Pr(\mathcal{A}_N(x)\mid x) = 1 - \mu_{\mathrm{pos}}(x)^N - (1 - \mu_{\mathrm{pos}}(x))^N$.

Let $Y_i^E = \mathbb{I}[o_i \in E_x]$, so $\Pr(Y_i^E = 1\mid x)=\tau_E(x)$. We are interested in the event that the update is active yet the target subset $E_x$ receives no samples:

\noindent
\begin{equation}
\mathcal{B}_{E,N}(x) := \mathcal{A}_N(x) \cap \Big\{\sum_{i=1}^N Y_i^E = 0\Big\}.
\label{eq:Btau_def}
\end{equation}

\begin{lemma}
\label{lem:P_Btau}
For any $N \geq 1$, writing $\mu_{\mathrm{pos}} = \mu_{\mathrm{pos}}(x)$ and $\tau = \tau_E(x)$ for brevity,
\begin{equation*}
\begin{aligned}
\Pr(\mathcal{B}_{E,N}(x)\mid x)
= {} &(1 - \tau)^N - (\mu_{\mathrm{pos}} - \tau)^N \\
&- (1 - \mu_{\mathrm{pos}})^N.
\end{aligned}
\label{eq:P_Btau}
\end{equation*}
\end{lemma}

The proof partitions rollouts into three disjoint regions and applies inclusion-exclusion (Appendix~\ref{app:proof_tailmiss}).

Lemma~\ref{lem:P_Btau} reveals a non-monotonic dependence on $N$ after conditioning on $x$. Two competing effects determine $\Pr(\mathcal{B}_{E,N}(x)\mid x)$: the coverage factor $(1 - \tau_E(x))^N$ decreases with $N$, improving the chance of sampling $E_x$, while activity $\Pr(\mathcal{A}_N(x)\mid x)$ increases from near zero toward one. Their interaction induces the qualitative picture in Figure~\ref{fig:main}(a) and Appendix Figure~\ref{fig:tailmiss_detailed}: small $N$ can preserve higher pass@$k$ through inactivity~\citep{it_takes_2, dang2025weight}, large $N$ through coverage~\citep{brorl}, while intermediate $N$ can yield frequent active updates that miss low-mass correct regions, consistent with distribution sharpening observations~\citep{rewarding_unlikely}. Because $\mu_{\mathrm{pos}}(x)$ and $\tau_E(x)$ vary across prompts and evolve during optimization, the result is most naturally read as a regime descriptor: it characterizes how inactivity, undercoverage, and coverage trade off as $N$ changes. Section~\ref{sec:results:scaling} later compares this qualitative picture with empirical group-size trends.

\subsection{Unsampled-correct mass under finite sampling}
\label{sec:theory:unsampled}

The tail-miss analysis identifies \emph{when} low-mass correct regions are vulnerable: active updates can occur while a target subset $E_x$ is absent from the sampled group. We now use the categorical framework (Section~\ref{sec:prelim:categorical}) to characterize the \emph{mechanism} by which unsampled positive mass can decrease.

While~\eqref{eq:delta_Qpos} shows that total correct mass $Q_{\mathrm{pos}}$ tends to increase with $N$ in the local categorical slice, it does not reveal redistribution within the correct set. Define the unsampled-correct mass

\noindent
\begin{equation}
Q_{\mathrm{u,pos}} := \sum_{i \in U \cap \mathcal{P}} p_i = Q_{\mathrm{pos}} - P_{\mathrm{pos}}.
\label{eq:Qupos_def}
\end{equation}
This quantity measures the positive probability mass outside the distinct sampled positive set in the categorical slice.

\begin{proposition}
\label{prop:delta_upos}
Under the one-step surrogate update~\eqref{eq:dz_update},
\noindent
\begingroup
\small
\begin{equation}
\begin{aligned}
&\Delta Q_{\mathrm{u,pos}} = \frac{\eta}{N} \Big[ 
\underbrace{-S_R\, U_{\mathrm{pos},2}}_{\text{direct drift}} \\[4pt]
&- Q_{\mathrm{u,pos}} \underbrace{\big( (R_c {-} S_R) S_{+,2} + (R_w {-} S_R) S_{-,2} - S_R U_2 \big)}_{\text{normalization coupling}}
\Big].
\end{aligned}
\label{eq:delta_Qupos}
\end{equation}
\endgroup
\end{proposition}

The proof applies the subset-mass identity from Appendix~\ref{app:first_order} with $\mathcal{S} = U \cap \mathcal{P}$; see Appendix~\ref{app:proof_upos} for details.

Equation~\eqref{eq:delta_Qupos} shows that $\Delta Q_{\mathrm{u,pos}}$ can be negative even when $\Delta Q_{\mathrm{pos}} > 0$: total positive mass can increase while unsampled positive mass decreases. This complements~\citet{brorl}, who showed that reward-positive batches ($S_R > 0$) push unsampled logits downward. Our formula makes explicit how this affects redistribution \emph{within} the correct set in the finite-action abstraction.

The mechanism operates through two terms. The \emph{direct drift} $-S_R U_{\mathrm{pos},2}$ pushes unsampled-correct mass downward when $S_R > 0$, with magnitude scaling with the concentration $U_{\mathrm{pos},2}$. The \emph{normalization coupling} (analyzed in detail in Appendix~\ref{app:term_analysis}) captures how probability gains by sampled-correct actions draw mass away from unsampled-correct ones through softmax normalization. In reward-positive batches, both terms can contribute negatively, so the update may reinforce the sampled correct subset while reducing probability on unsampled correct actions.

As \citet{brorl} observe, scaling $N$ suppresses $U_{\mathrm{pos},2}$ and ensures $\Delta Q_{\mathrm{pos}} \geq 0$ with the \textit{direct drift} term tending to zero. However, practical constraints limit how far $N$ can be scaled: computational cost grows linearly with $N$, and improving pass@1 requires active groups (ruling out very small $N$ where most groups are homogeneous). This yields the intermediate-$N$ risk from Section~\ref{sec:theory:tailmiss}: active updates that still miss low-probability correct regions and then pull mass toward the sampled correct subset. The categorical simulation in Section~\ref{sec:categorical_sim} measures this redistribution directly through retained positive mass.

\section{F-GRPO: Focal weighting for Group-Relative Policy Optimization}
\label{sec:method}

The finite-sampling effects analyzed in Section~\ref{sec:theory} are prompt-local. For each prompt, rewards are assigned to sampled complete rollouts and group-relative coefficients are computed from the same finite group. Consequently, direct positive evidence enters the update only through correct rollouts that appear in the group, while unsampled correct behavior is affected only indirectly. F-GRPO follows this group granularity by multiplying the group-relative signal with a single group-dependent weight.

The rare correct subset $E_x$ and the unsampled-correct mass are latent during training, so the analysis does not define a directly usable group weight. The available group-level statistic is the number of correct rollouts $X$, equivalently the empirical success rate $\widehat{\mu}_{\mathrm{pos}}(x)=X/N$. We use this statistic as an observable proxy for sampled positive exposure. A small nonzero value indicates that correct behavior has appeared but remains sparsely observed, whereas a large value indicates that the current update is supported by many positive rollouts. The resulting weighting attenuates high-success groups more strongly and attenuates sparse-positive mixed groups less.

\subsection{Focal Weight}

Define the empirical success rate for prompt $x$ as

\noindent
\begin{equation}
\widehat{\mu}_{\mathrm{pos}}(x) := \frac{\bar{R}(x) - R_w}{R_c - R_w} = \frac{X}{N} \in [0,1],
\label{eq:muhat_method}
\end{equation}
where $X$ is the number of correct rollouts and $\bar{R}(x) = \frac{1}{N}\sum_{i=1}^N R_i$ is the group mean reward. This is an unbiased estimator of the true success probability: $\mathbb{E}[\widehat{\mu}_{\mathrm{pos}}(x)] = \mu_{\mathrm{pos}}(x)$.

Under binary rewards, \(\widehat{\mu}_{\mathrm{pos}}(x)=X/N\) is a convenient observable summary of how much correct behavior the sampled group has already exposed for prompt $x$. It distinguishes groups where correct rollouts are still rare from groups where the sampled set already contains many correct rollouts. Appendix~\ref{app:proof_sr_monotone} further shows that the expected distinct sampled-correct mass \(\mathbb E[P_{\mathrm{pos}}\mid X=k]\) is non-decreasing in \(k\), so larger \(\widehat{\mu}_{\mathrm{pos}}(x)\) corresponds in expectation to broader sampled-correct exposure; in the categorical picture, the expected \(S_R\) moves in the same direction.

Using this proxy, we apply a Focal-style transformation~\citep{focal} that reduces the contribution of high-success groups:
\noindent
\begin{equation}
g(x) := \big(1 - \widehat{\mu}_{\mathrm{pos}}(x)\big)^\gamma, \quad \gamma \geq 0.
\label{eq:g_method}
\end{equation}
When $\gamma = 0$, $g(x) = 1$ for all prompts, recovering standard GRPO. For $\gamma > 0$, prompts with high empirical success rate receive reduced weight: $g(x) \to 0$ as $\widehat{\mu}_{\mathrm{pos}}(x) \to 1$.

With binary rewards, the GRPO advantage magnitudes vary with $\mu_{\mathrm{pos}}(x)$. The Focal weight $g(x) = (1 - \widehat{\mu}_{\mathrm{pos}}(x))^\gamma$ scales these magnitudes, attenuating high-success groups more strongly than mixed low-success ones. Appendix~\ref{app:focal_visualization} visualizes this effect.

\begin{figure*}[t]
\centering
\includegraphics[width=\textwidth]{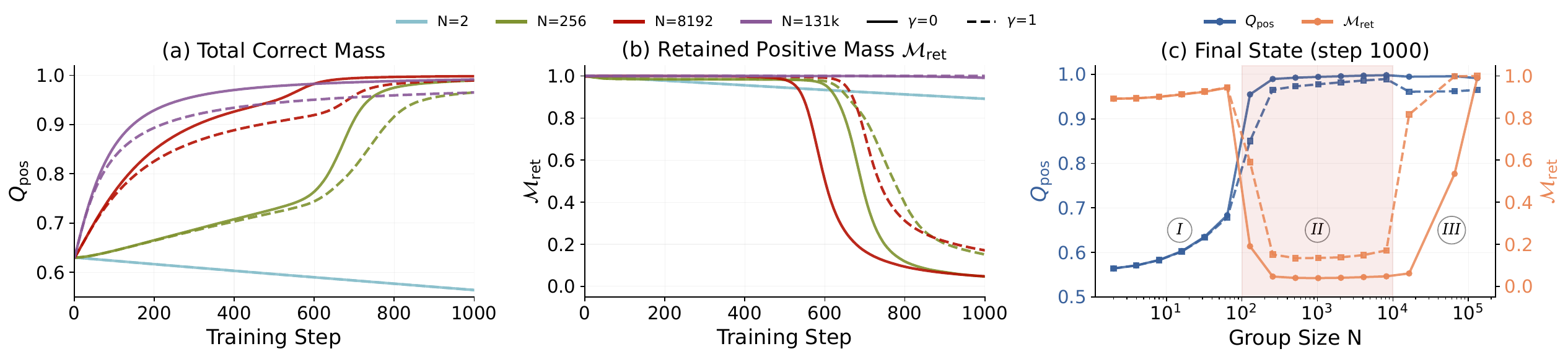}
\caption{Categorical policy simulation following \citet{brorl} setup. \textbf{(a)}~Total correct mass $Q_{\mathrm{pos}}$ vs.\ training step. \textbf{(b)}~Retained positive mass $\mathcal{M}_{\mathrm{ret}}$ vs.\ step. \textbf{(c)}~Final metrics vs.\ group size $N$, with three regimes: \textbf{I}~slow $Q_{\mathrm{pos}}$ growth, positive mass retained; \textbf{II}~concentration zone (shaded), $Q_{\mathrm{pos}}$ grows but $\mathcal{M}_{\mathrm{ret}}$ collapses; \textbf{III}~both metrics high. Solid: $\gamma{=}0$; dashed: $\gamma{=}1$. $N{=}131\text{k}$ maintains $\mathcal{M}_{\mathrm{ret}}{\approx}1$ throughout, consistent with $\Pr(\mathcal{B}_{E,N}(x)\mid x)<10^{-3}$ for a non-anchor correct action under the initial distribution (Appendix~\ref{app:categorical_sim}).}
\label{fig:categorical_sim}
\vspace{-0.5em}
\end{figure*}

\subsection{Integration with Group-Relative Methods}

We incorporate the Focal weight by scaling the group-relative advantage:
\noindent
\begin{equation}
\widehat{A}_i^{\mathrm{F-GRPO}} := g(x) \cdot \widehat{A}_i^{\mathrm{GRPO}}.
\label{eq:da_advantage}
\end{equation}
The Focal weight is orthogonal to the clipping and importance-weighting choices in the base optimizer: it multiplies the group-level advantage without changing reward values, clipping bounds, or importance-weight construction. It can therefore be combined with objectives that differ in these details. We denote the corresponding variants as F-DAPO and F-CISPO.

The modification is minimal: a single scalar $g(x) \in [0,1]$ applied uniformly to all rollouts from the same prompt. No additional networks are required; $\gamma$ is the only new hyperparameter.

\section{Experiments \& Results}
\label{sec:exps_results}

We evaluate the paper's claims in four complementary settings. We first test the local redistribution mechanism in a closed categorical simulation where the relevant quantities are directly observable. We then use single-solution Maze to ask whether evaluation degradation can appear across prompts even when each prompt has only one correct trajectory, and compare the resulting picture with one representative GRPO group-size sweep in LLM training. Finally, we test whether the proposed fixed-$N$ reweighting transfers across GRPO-family methods and remains distinct from generic regularization or simple update-scale reduction.

\subsection{Empirical Validation via Categorical Simulation}
\label{sec:categorical_sim}

To complement the simulation analysis of \citet{brorl}, we conduct experiments under the same categorical policy framework (Section~\ref{sec:prelim:categorical}) with an additional focus on \emph{which} correct actions retain probability mass. This is the closed-world setting where the local quantities from Section~\ref{sec:theory:unsampled} are directly observable, so it directly tests the categorical mechanism: whether total positive mass can grow while retained positive mass collapses. Following \citet{brorl}, we simulate a softmax policy over $128{,}000$ actions ($10{,}000$ correct) trained with group-relative updates; see Appendix~\ref{app:categorical_sim} for details.

Beyond tracking total correct mass $Q_{\mathrm{pos}}$, we track $\mathcal{M}_{\mathrm{ret}}(t)$, the \emph{retained positive mass}, which measures the fraction of initial correct-action probability that remains at or above its starting value (Appendix~\ref{app:categorical_sim} Eq.~\ref{eq:retained_mass}). Values near $1$ indicate that initially positive actions have not lost probability mass; values near $0$ indicate concentration onto a smaller subset of correct actions.

Figure~\ref{fig:categorical_sim} presents the results. Panel~(a) confirms that $Q_{\mathrm{pos}}$ increases for all group sizes, consistent with \citet{brorl}. However, panel~(b) shows that $\mathcal{M}_{\mathrm{ret}}$ behaves non-monotonically: both small and large $N$ retain more positive mass, while intermediate values suffer severe concentration. This demonstrates that $\Delta Q_{\mathrm{pos}} > 0$ does not guarantee preservation of unsampled correct actions. Panel~(c) summarizes the final state across all group sizes, with three regimes labeled: \textbf{(I)}~small $N$ where $Q_{\mathrm{pos}}$ grows slowly but positive mass is retained; \textbf{(II)}~the concentration zone (shaded) where $Q_{\mathrm{pos}}$ grows rapidly but $\mathcal{M}_{\mathrm{ret}}$ collapses; and \textbf{(III)}~large $N$ where both metrics are high. Notably, $N{=}131{,}072$ maintains $\mathcal{M}_{\mathrm{ret}} \approx 1$ throughout training, consistent with Lemma~\ref{lem:P_Btau}, which predicts $\Pr(\mathcal{B}_{E,N}(x)\mid x)<10^{-3}$ for a non-anchor correct action under the initial distribution (see Appendix~\ref{app:categorical_sim}). Dashed lines ($\gamma{=}1$) show improved $\mathcal{M}_{\mathrm{ret}}$ retention, particularly in the concentration zone.

The specific boundaries of the concentration zone depend on the initial distribution and should not be interpreted as quantitative predictions for LLM training. The key insight is the qualitative pattern: intermediate group sizes can yield the lowest retained positive mass.

\begin{table*}[t!]
  \centering
  \scriptsize
  \setlength{\tabcolsep}{2.2pt}
  \setlength{\extrarowheight}{-0.0pt}
  \resizebox{\textwidth}{!}{%
  \begin{tabular}{l|c|cccccc|c|ccc}
  \Xhline{1.1pt}
  \rule{0pt}{2.0ex} & \multicolumn{7}{c|}{\textbf{In-domain}} & \multicolumn{4}{c}{\textbf{Out-of-domain}} \\
  \hhline{~|-|------|-|---}
  \rule{0pt}{2.0ex} \textbf{Method} & \textbf{Avg.} & \textbf{AIME24} & \textbf{AIME25} & \textbf{AMC} & \textbf{MATH500} & \textbf{Minerva} & \textbf{Olympiad} & \textbf{Avg. OOD} & \textbf{IFEval} & \textbf{SynLogic} & \textbf{GPQA}  \\
  \Xhline{1.1pt}
  \multicolumn{12}{c}{\rule{0pt}{2.0ex}\textbf{\textit{Qwen2.5-7B}}} \\[0.2ex]
  \hhline{-|-|------|-|---}
  \rule{0pt}{2.0ex}GRPO   & 37.3/64.1 & 15.0/37.7 & 6.7/40.8 & 52.9/87.3 & 75.8/92.8 & \textbf{36.0}/60.2 & \textbf{37.8}/65.8 & 17.1/55.9 & 32.1/70.3 & 7.9/51.3 & 11.3/46.2 \\
  \rowcolor{gray!15}
  F-GRPO  & \underline{\textbf{38.6}}/\underline{\textbf{70.3}} & \underline{\textbf{15.9}}/\textbf{46.2} & \underline{\textbf{10.1}}/\underline{\textbf{52.6}} & \underline{\textbf{56.2}}/\underline{\textbf{96.3}} & \underline{\textbf{76.2}}/\underline{\textbf{95.1}} & 35.7/\textbf{60.3} & 37.5/\underline{\textbf{71.6}} & \underline{\textbf{19.2}}/\underline{\textbf{63.3}} & \underline{\textbf{34.0}}/\underline{\textbf{75.7}} & \underline{\textbf{8.7}}/\underline{\textbf{57.0}} & \underline{\textbf{15.0}}/\underline{\textbf{57.3}} \\
  \rule{0pt}{2.5ex}DAPO   & 39.4/69.3 & 16.8/49.8 & \textbf{12.0}/45.6 & 53.3/91.9 & 78.6/95.2 & \textbf{35.5}/61.2 & 40.5/71.8 & 15.7/58.4 & 24.1/67.1 & 7.5/53.3 & \textbf{15.4}/54.9 \\
  \rowcolor{gray!15}
  F-DAPO  & \underline{\textbf{40.5}}/\underline{\textbf{72.5}} & \underline{\textbf{20.9}}/\textbf{53.4} & 11.5/\textbf{52.9} & \underline{\textbf{55.9}}/\textbf{93.7} & \underline{\textbf{79.1}}/\underline{\textbf{96.6}} & 35.0/\textbf{62.9} & \underline{\textbf{40.9}}/\underline{\textbf{75.6}} & \underline{\textbf{17.9}}/\underline{\textbf{63.6}} & \underline{\textbf{30.8}}/\underline{\textbf{71.1}} & \underline{\textbf{7.9}}/\underline{\textbf{62.4}} & 15.0/\textbf{57.4} \\
  
  \rule{0pt}{2.5ex}CISPO  & 39.5/73.2 & 14.6/45.9 & 9.7/59.8 & \textbf{57.8}/96.1 & 78.7/97.0 & \textbf{34.7}/63.3 & 41.5/76.9 & 14.9/59.0 & 24.2/67.9 & 8.0/53.6 & 12.6/55.5 \\
  \rowcolor{gray!15}
  F-CISPO & 39.5/\underline{\textbf{76.8}} & \textbf{14.8}/\underline{\textbf{59.7}} & \underline{\textbf{13.0}}/\textbf{64.6} & 53.3/\textbf{97.1} & \underline{\textbf{79.0}}/\textbf{97.8} & 34.6/\textbf{64.3} & \underline{\textbf{42.4}}/\textbf{77.5} & \underline{\textbf{18.1}}/\underline{\textbf{65.9}} & \underline{\textbf{30.7}}/\underline{\textbf{70.6}} & \textbf{8.2}/\underline{\textbf{60.0}} & \underline{\textbf{15.4}}/\underline{\textbf{67.1}} \\
  \Xhline{0.8pt}
  \multicolumn{12}{c}{\rule{0pt}{2.0ex}\textbf{\textit{Qwen3-4B-Base}}} \\[0.2ex]
  \hhline{-|-|------|-|---}
  \rule{0pt}{2.0ex}GRPO   & 39.9/71.1 & 14.8/52.5 & 8.3/45.6 & 54.6/96.2 & 79.5/96.4 & \textbf{39.2}/61.7 & 43.1/74.2 & 22.2/67.9 & 36.9/91.2 & 10.7/57.1 & 19.0/55.4 \\
  \rowcolor{gray!15}
  F-GRPO  & \underline{\textbf{42.7}}/\underline{\textbf{76.0}} & \underline{\textbf{17.4}}/\textbf{61.1} & \underline{\textbf{15.8}}/\underline{\textbf{60.0}} & \underline{\textbf{58.3}}/\textbf{98.2} & \underline{\textbf{82.3}}/\textbf{97.4} & 38.5/\textbf{63.6} & \underline{\textbf{43.7}}/\textbf{75.5} & \underline{\textbf{24.3}}/\underline{\textbf{73.6}} & \underline{\textbf{41.2}}/\underline{\textbf{93.6}} & 10.7/\underline{\textbf{64.5}} & \underline{\textbf{21.1}}/\underline{\textbf{62.8}} \\
  
  \rule{0pt}{2.5ex}DAPO   & 45.1/76.1 & 18.6/\textbf{63.8} & 17.6/54.0 & 62.0/96.2 & 84.4/\textbf{97.8} & \textbf{40.0}/67.1 & 47.8/77.9 & 20.7/74.4 & 33.8/88.0 & \textbf{11.8}/71.3 & 16.4/63.9 \\
  \rowcolor{gray!15}
  F-DAPO  & \underline{\textbf{45.3}}/\textbf{76.4} & \textbf{19.1}/59.7 & \textbf{17.7}/\textbf{58.3} & \underline{\textbf{63.0}}/\textbf{96.8} & 84.4/97.7 & 39.5/\textbf{67.9} & \underline{\textbf{48.2}}/\textbf{78.0} & \underline{\textbf{23.2}}/\underline{\textbf{76.8}} & \underline{\textbf{39.6}}/\underline{\textbf{90.6}} & 11.3/\textbf{72.4} & \underline{\textbf{18.7}}/\underline{\textbf{67.5}} \\
  
  \rule{0pt}{2.5ex}CISPO  & 45.8/75.3 & 21.5/59.3 & 17.6/55.4 & \textbf{63.1}/97.4 & 84.2/97.9 & 40.2/65.1 & 48.4/76.6 & 22.6/76.5 & 32.6/85.9 & \textbf{11.6}/\textbf{71.1} & 23.6/72.4 \\
  \rowcolor{gray!15}
  F-CISPO & \textbf{46.0}/\underline{\textbf{79.2}} & 21.5/\textbf{69.5} & \underline{\textbf{18.6}}/\textbf{62.7} & 60.7/\textbf{98.6} & 84.2/\textbf{98.2} & \underline{\textbf{42.1}}/\underline{\textbf{67.2}} & \underline{\textbf{48.9}}/\underline{\textbf{78.8}} & \underline{\textbf{23.9}}/\underline{\textbf{80.0}} & \underline{\textbf{35.9}}/\underline{\textbf{90.3}} & 11.1/68.0 & \underline{\textbf{24.6}}/\underline{\textbf{81.6}} \\
  \Xhline{0.8pt}
  \multicolumn{12}{c}{\rule{0pt}{2.0ex}\textbf{\textit{Llama-3.2-3B-Instruct}}} \\[0.2ex]
  \hhline{-|-|------|-|---}
  \rule{0pt}{2.0ex}GRPO   & 23.0/59.9 & 10.7/40.7 & 0.7/21.5 & \textbf{30.5}/88.2 & \textbf{55.0}/90.6 & \textbf{21.8}/59.0 & 19.4/59.3 & \textbf{25.5}/56.5 & 54.1/78.0 & \textbf{4.7}/\textbf{36.4} & \textbf{17.5}/55.1 \\
  \rowcolor{gray!15}
  F-GRPO  & 23.0/\underline{\textbf{63.4}} & \underline{\textbf{12.1}}/\textbf{46.1} & \underline{\textbf{1.0}}/\textbf{29.5} & 29.8/\textbf{90.6} & 54.1/\underline{\textbf{92.9}} & 21.0/\textbf{60.1} & \underline{\textbf{20.1}}/\underline{\textbf{61.3}} & 25.4/\textbf{57.6} & \underline{\textbf{56.4}}/\underline{\textbf{79.6}} & 4.6/35.5 & 15.2/\textbf{57.6} \\
  
  \rule{0pt}{2.5ex}DAPO   & 24.3/54.2 & \textbf{12.8}/40.8 & 1.0/18.5 & \textbf{33.1}/79.5 & 55.9/83.8 & \textbf{22.4}/54.1 & 21.0/48.4 & 23.9/51.3 & 51.2/77.8 & \textbf{4.8}/28.9 & 15.7/47.0 \\
  \rowcolor{gray!15}
  F-DAPO  & \underline{\textbf{24.8}}/\underline{\textbf{62.3}} & 11.1/\textbf{44.4} & \underline{\textbf{1.7}}/\textbf{28.7} & 31.9/\underline{\textbf{88.3}} & \underline{\textbf{58.6}}/\underline{\textbf{92.0}} & 22.3/\underline{\textbf{59.3}} & \underline{\textbf{23.2}}/\underline{\textbf{61.3}} & \underline{\textbf{24.8}}/\underline{\textbf{55.4}} & \underline{\textbf{53.0}}/\underline{\textbf{79.5}} & 4.3/\underline{\textbf{33.0}} & \underline{\textbf{17.0}}/\underline{\textbf{53.7}} \\
  
  \rule{0pt}{2.5ex}CISPO  & 24.1/58.0 & 9.7/39.4 & 1.0/\textbf{25.4} & 32.9/79.1 & \textbf{56.9}/89.1 & 21.8/\textbf{59.5} & \textbf{22.5}/55.4 & \textbf{25.7}/52.5 & \textbf{54.6}/\textbf{78.4} & 4.3/29.4 & \textbf{18.2}/\textbf{49.7} \\
  \rowcolor{gray!15}
  F-CISPO & \underline{\textbf{24.5}}/\textbf{59.7} & \underline{\textbf{10.6}}/\textbf{42.8} & \underline{\textbf{2.0}}/24.5 & \underline{\textbf{34.1}}/\textbf{82.6} & 56.5/\underline{\textbf{91.0}} & \underline{\textbf{22.1}}/58.8 & 21.5/\underline{\textbf{58.7}} & 25.0/\textbf{53.0} & 52.6/77.3 & \underline{\textbf{5.4}}/\underline{\textbf{33.9}} & 17.0/47.7 \\
  \Xhline{1.1pt}
  \end{tabular}%
  }
  \caption{Pass@1 / pass@256 across three models and six methods at $N{=}8$. Across all nine baseline/Focal pairs, Focal weighting improves Avg. Math and Avg. OOD pass@256 while preserving or improving Avg. Math pass@1. \textbf{Bold}: better within baseline/Focal pair. \underline{Underline}: statistically significant under paired subsampling ($p{<}0.05$; Appendix~\ref{app:statsign}).}
  \label{tab:all_models}
  \vspace{-1.5em}
  \end{table*}

\subsection{Single-Solution Maze Experiments}
\label{sec:maze}

We adapt the maze-navigation setup of \citet{maxrl} to single-solution scoring, using the shortest-path action sequence as the target for each prompt. We compare GRPO and F-GRPO with $\gamma\in\{0.5,1.0\}$ across $N\in\{4,8,16,32,64,128\}$ after SFT initialization; full details are in Appendix~\ref{app:maze}.

Appendix Figure~\ref{fig:maze_passk_vs_n} shows that F-GRPO improves final pass@1 and pass@$k$, and reduces the early training-time drop in pass@$k$. At final evaluation, F-GRPO with $\gamma{=}0.5$ improves pass@1 for every $N$ (74.4--93.6 vs.\ 65.6--75.8 for GRPO) and pass@256 (75.9--96.5 vs.\ 67.3--76.9). Because each prompt has a unique target sequence, the observed large-$k$ degradation reflects reduced target-sequence coverage at the evaluation-distribution level rather than loss of alternative correct trajectories within a prompt. Thus Maze separates the pass@$k$ effect from within-prompt multi-solution diversity and shows that the same weighting also mitigates degradation in a single-solution setting.

\subsection{LLM Experimental Setup}
\label{sec:exp_setup}

We train Qwen2.5-7B~\citep{qwen2.5}, Qwen3-4B-Base~\citep{qwen3}, and Llama-3.2-3B-Instruct~\citep{llama} on DeepScaleR~\citep{deepscaler}, and Qwen3-1.7B on SynLogic~\citep{synlogic}, using verl~\citep{verl}. For the math experiments, we report pass@1 and pass@256 on in-domain benchmarks (MATH500~\citep{math500}, AIME24/25~\citep{aime2425}, AMC23~\citep{amc23}, Minerva Math~\citep{minerva}, Olympiad Bench~\citep{olympiad}) and OOD benchmarks (GPQA Diamond~\citep{gpqa_diamond}, IFEval~\citep{ifeval}, SynLogic). For Qwen3-1.7B, SynLogic is the sole in-domain benchmark, and the OOD suite comprises AIME24, AIME25, AMC23, MATH500, IFEval, and GPQA Diamond. Full training, evaluation, and \(\gamma\)-selection details are in Appendix~\ref{app:training}; additional pass@$k$ tables and training dynamics are in Appendices~\ref{app:llm_passk_curves} and~\ref{app:llm_training_dynamics}.

\subsection{Group Size Regimes and Focal Weighting}
\label{sec:results:scaling}

\begin{table}[b!]
\centering
\small
\setlength{\tabcolsep}{3.2pt}
\newcommand{\pmsep}{\hspace{0.55mm}}
\begin{tabular}{lccc}
\toprule
\textbf{Method} & \textbf{Avg. Math} $\uparrow$ & \textbf{Avg. OOD} $\uparrow$ & $\Delta$\textbf{NLL}$_{\text{rare}}$ $\downarrow$ \\
\midrule
GRPO $N{=}2$ & 36.2\pmsep/\pmsep\textbf{75.0} & 18.0\pmsep/\pmsep\textbf{67.3} & \textbf{0.19} \\
GRPO $N{=}4$ & 36.4\pmsep/\pmsep{71.1} & 18.7\pmsep/\pmsep 59.7 & 0.44 \\
\rowcolor{gray!15}
GRPO $N{=}8$ & 37.3\pmsep/\pmsep 64.1 & 17.1\pmsep/\pmsep 55.9 & 0.68 \\
GRPO $N{=}16$ & 38.4\pmsep/\pmsep 67.5 & 17.7\pmsep/\pmsep 55.7 & 0.66 \\
GRPO $N{=}32$ & \textbf{39.2}\pmsep/\pmsep 70.1 & 17.7\pmsep/\pmsep 61.7 & 0.52 \\
\midrule
\rowcolor{gray!15}
F-GRPO $N{=}8$ & 38.6\pmsep/\pmsep 70.3 & \textbf{19.2}\pmsep/\pmsep 63.3 & 0.46 \\
\bottomrule
\end{tabular}
\caption{
Comparison of GRPO at varying group sizes versus F-GRPO at fixed $N{=}8$ on Qwen2.5-7B. Metrics: average pass@1 / pass@256 on in-domain math and OOD benchmarks. $\Delta$NLL$_{\text{rare}}$: diagnostic increase in negative log-likelihood on a fixed set of correct trajectories that were rare under the base model by NLL (lower $=$ less deviation from base distribution; see Appendix~\ref{app:nll_rare}). \textbf{Bold}: best. Gray rows mark the same group size $N{=}8$ comparison. Full per-benchmark results in Appendix~\ref{app:group_size_results}.
}
\label{tab:grpo_n_summary}
\end{table}

Having validated the local categorical mechanism and used Maze as a single-solution setting, we next examine empirical group-size trends in LLM training. Table~\ref{tab:grpo_n_summary} compares GRPO at $N \in \{2,4,8,16,32\}$ with F-GRPO at $N=8$ on Qwen2.5-7B. These group sizes are chosen to probe small, intermediate, and larger rollout regimes while keeping rollout cost tractable; we do not aim to exhaustively map performance as a function of $N$.

GRPO exhibits non-monotonic pass@256 behavior across group sizes: $N{=}2$ yields high pass@256 but lowest pass@1, a pattern consistent with slow policy change under infrequent active updates. Moving to intermediate group sizes improves math pass@1 but reduces pass@256; the degradation is strongest around $N{=}8$ on in-domain benchmarks and spans $N{=}8$-$16$ on OOD benchmarks. At $N{=}32$, pass@256 partially recovers while math pass@1 continues to improve. This pattern is qualitatively consistent with an intermediate concentration regime between small-$N$ inactivity and larger-$N$ coverage. The local theory does not predict the exact boundaries for LLM training, but it does motivate this qualitative comparison.

At $N{=}8$, F-GRPO matches GRPO at $N{=}32$ on pass@256 (70.3 vs.\ 70.1 on math; 63.3 vs.\ 61.7 on OOD) using $4{\times}$ fewer rollouts. Pass@1 shows a modest trade-off on in-domain benchmarks but improves on OOD tasks, suggesting that Focal weighting can recover much of the pass@256 loss observed under GRPO at intermediate $N$ without increasing the rollout budget in this setting.

\textbf{Deviation from Base-Model Rare Solutions.}
We also report $\Delta$NLL$_{\text{rare}}$, computed on correct trajectories that were rare under the base model by NLL. Higher values indicate greater deviation from the base distribution on these trajectories. The ordering $\Delta$NLL$_{\text{rare}}(N{=}2) < \Delta$NLL$_{\text{rare}}(N{=}32) < \Delta$NLL$_{\text{rare}}(N{=}8)$ is qualitatively consistent with the pass@256 degradation pattern, with F-GRPO at $N{=}8$ achieving an intermediate value (0.46), indicating less deviation on these base-model-rare correct trajectories than its baseline. Appendix~\ref{app:nll_rare} also compares the summed full-trajectory probability that GRPO and F-GRPO assign to fixed rare subsets.

\textbf{Correct-Solution Diversity.}
At matched $N{=}8$, F-GRPO has higher mean within-prompt embedding distance than GRPO on all six math benchmarks and higher GPT-5.1 mathematical-strategy diversity scores on all four judged benchmarks. More details are provided in Appendix~\ref{app:solution_diversity}.

\subsection{Focal Weighting Across Methods}
\label{sec:results:methods}

We now ask a more practical question: whether the same fixed-$N$ reweighting transfers across GRPO-family optimizers. Table~\ref{tab:all_models} reports results for GRPO, DAPO, and CISPO at $N{=}8$, a commonly used group size~\citep{grpo, rl_zoo, trick_part1}, across three models at different scales. On Qwen2.5-7B, pass@256 gains are $+6.2$ (GRPO), $+3.2$ (DAPO), and $+3.6$ (CISPO); corresponding gains are $+4.9$/$+0.3$/$+3.9$ on Qwen3-4B-Base and $+3.5$/$+8.1$/$+1.7$ on Llama-3.2-3B-Instruct. Across all nine method-model combinations, Focal weighting improves both math and OOD pass@256 (average $+3.9$ and $+4.1$). Math pass@1 is preserved or improved in 9/9 cases, with gains up to $+2.8$, and OOD pass@1 improves in 7/9 cases (average $+1.5$).

Across the nine math-trained settings, $\gamma{=}0.5$ was selected in 7/9 cases and $\gamma{=}1$ in 2/9. We recommend starting with $\gamma{=}0.5$; larger values may improve pass@256 at the cost of pass@1 (Appendix~\ref{app:gamma_sweep}).

In the SynLogic setup, Focal weighting also improves all four reported metrics for GRPO, DAPO, and CISPO. Full results are in Table~\ref{tab:synlogic_full}.

\begin{table}[h]
\centering
\small
\setlength{\tabcolsep}{4pt}
\begin{tabular}{lcc}
\toprule
\textbf{Method} & \textbf{SynLogic (ID)} & \textbf{Avg. OOD} \\
\midrule
GRPO & 40.1 / 74.8 & 31.7 / 60.6 \\
\rowcolor{gray!15}
F-GRPO & \textbf{41.8 / 79.4} & \textbf{35.8 / 63.9} \\
\midrule
DAPO & 57.1 / 86.3 & 38.7 / 67.1 \\
\rowcolor{gray!15}
F-DAPO & \textbf{58.6 / 87.2} & \textbf{39.1 / 69.4} \\
\midrule
CISPO & 45.5 / 78.5 & 36.2 / 63.4 \\
\rowcolor{gray!15}
F-CISPO & \textbf{51.0 / 85.2} & \textbf{37.4 / 64.2} \\
\bottomrule
\end{tabular}
\caption{Qwen3-1.7B trained on SynLogic at $N{=}8$. Metrics are pass@1 / pass@256. \textbf{Bold}: better within each baseline/Focal pair.}
\label{tab:synlogic_summary}
\end{table}

\subsection{Comparison with Regularization and Update-Scale Controls}
\label{sec:grpo_controls}

We include the base model as a reference and compare F-GRPO against GRPO with a lower learning rate (GRPO low-LR), GRPO with entropy bonus (GRPO-$\mathcal{H}$), GRPO with KL penalty (GRPO-KL), Rewarding the Unlikely (RUL; \citealp{rewarding_unlikely}), DARO~\citep{daro}, and differential smoothing (DS-GRPO; \citealp{diff_smooth_sharp}) using the Qwen2.5-7B setup. GRPO low-LR checks whether the average scale reduction alone is sufficient by matching the average advantage-magnitude scale induced by Focal weighting. For $\gamma=0.5$, this gives $\mathrm{lr}_{\mathrm{low}}=6.8{\times}10^{-7}$; coefficient tuning and the multiplier derivation are in Appendices~\ref{app:training} and~\ref{app:grpo_controls}.

\begin{table}[H]
\centering
\small
\setlength{\tabcolsep}{3.2pt}
\begin{tabular}{lcc}
\toprule
\textbf{Method} & \textbf{Avg. Math} $\uparrow$ & \textbf{Avg. OOD} $\uparrow$ \\
\midrule
Base & 10.9 / 70.5 & 9.3 / \textbf{70.7} \\
\midrule
GRPO low-LR & \utilde{37.8} / 69.2 & 16.4 / 57.9 \\
GRPO-$\mathcal{H}$ & \utilde{37.8} / 69.5 & 18.7 / 59.9 \\
GRPO-KL & 37.2 / \utilde{72.0} & \textbf{19.4} / 60.0 \\
\midrule
RUL & 37.1 / 69.1 & 18.1 / 61.7 \\
DARO & 37.3 / 69.4 & 18.8 / 63.3 \\
DS-GRPO & 37.7 / \textbf{73.8} & 17.9 / \utilde{68.3} \\
\rowcolor{gray!15}
F-GRPO & \textbf{38.6} / 70.3 & \utilde{19.2} / 63.3 \\
\bottomrule
\end{tabular}
\caption{Qwen2.5-7B comparison; trained methods use $N{=}8$. Metrics are average pass@1 / pass@256. \textbf{Bold}: best; $\dagger$: second best.}
\label{tab:controls_summary}
\end{table}

F-GRPO achieves the highest math pass@1 ($38.6$) and the second-highest OOD pass@1 ($19.2$, behind GRPO-KL at $19.4$, which incurs reference-model overhead). DS-GRPO has higher pass@256 than F-GRPO ($73.8$ vs.\ $70.3$ on math; $68.3$ vs.\ $63.3$ on OOD), but lower pass@1 ($37.7$ vs.\ $38.6$ on math; $17.9$ vs.\ $19.2$ on OOD). This is a pass@1/pass@256 trade-off rather than uniform dominance. GRPO low-LR improves pass@256 over default GRPO but remains below F-GRPO, suggesting that average scale reduction alone does not fully account for the gains. In this setup, the controls suggest that the gains are not explained solely by generic regularization or simple step-size reduction.

\section{Related Work}
\label{sec:related}

\textbf{Distribution Sharpening and Group Size.}
A growing body of work documents that RLVR improves pass@1 while degrading pass@$k$ for large $k$, indicating concentration onto fewer solutions~\citep{dang2025weight, rl_passk_incentivize, wu2025invisible}. The optimal rollout count remains debated: \citet{it_takes_2} show that $N{=}2$ is theoretically justified and compute-efficient, while \citet{brorl} advocate large groups for coverage. We provide a complementary perspective: finite-sampling effects in group-relative methods, including rare-mode undercoverage within prompts and success-based reweighting during training.

\textbf{Difficulty-Aware, Entropy, and Token-level Approaches.}
Related approaches include difficulty-aware reweighting~\citep{focal, curriculum, curriculum_rlvr, daro, your_gr_biased}, rare-trajectory rewards~\citep{rewarding_unlikely}, differential smoothing~\citep{diff_smooth_sharp}, entropy interventions~\citep{entrmech, cheng2025reasoning, agarwal2025unreasonable}, and token-level concentration methods~\citep{rethinking, simko, wang2025beyond}. Due to page-limit constraints, these approaches are discussed in Appendix~\ref{app:related_work}. Our focus here is on finite-group sampling in group-relative RLVR, its dependence on rollout count \(N\), and success-based weighting of whole sampled groups.

\section{Conclusion}
\label{sec:con}

This work studies finite-group RLVR at the prompt-local level where group-relative updates are actually formed. Our theoretical analysis derives a closed-form tail-miss probability exhibiting non-monotonic dependence on $N$: small groups limit active misses through inactivity, large groups through coverage, but intermediate $N$ maximize the probability of active misses. In the categorical abstraction, we further characterize redistribution within the correct set, showing that unsampled-correct mass can shrink even as total correct mass grows. Empirically, categorical simulation illustrates this mechanism in the categorical setting; a single-solution Maze setting shows that evaluation pass@$k$ can degrade across prompts even when each prompt has one correct trajectory; and a representative GRPO LLM sweep is qualitatively consistent with the same regime picture. Motivated by this analysis, Focal weighting provides a lightweight fixed-$N$ mitigation that improves pass@256 across GRPO, DAPO, and CISPO while preserving or improving pass@1 in most reported settings, with especially consistent OOD gains and no extra rollout cost.

\section*{Limitations}

Our theoretical analysis isolates one finite-sampling mechanism in group-relative RLVR with binary outcome rewards, namely that active groups can miss low-mass correct regions and then reallocate probability toward the sampled correct subset. The resulting claim is local, and the categorical simulation tests that local redistribution mechanism directly in a controlled setting. In this framing, the Maze experiments ask whether pass@$k$ degradation can arise even when each prompt has a single correct trajectory, and the LLM group-size sweep asks only whether qualitatively similar trends appear in practical training. F-GRPO is evaluated on that basis as a fixed-$N$ mitigation motivated by this mechanism, not as a general rule for RLVR concentration or group-size selection.

\bibliography{custom}

\clearpage
\appendix

\section*{Appendix Contents}
\begingroup
\small
\begin{enumerate}[label=\textbf{\Alph*.}, leftmargin=2.1em, itemsep=0pt, topsep=2pt]
\item \hyperref[app:related_work]{Additional related work}
\item \hyperref[app:objective_details]{Group-relative objective details}
\item \hyperref[app:categorical_prelim]{Categorical total-mass expression}
\item \hyperref[app:proof_tailmiss]{Proof of the tail-miss probability}
\item \hyperref[app:first_order]{First-order softmax expansion}
\item \hyperref[app:proof_upos]{Proof for unsampled-correct mass}
\item \hyperref[app:term_analysis]{Detailed term analysis}
\item \hyperref[app:proof_sr_monotone]{Monotonicity of sampled distinct mass}
\item \hyperref[app:focal_visualization]{Focal-weight visualization}
\item \hyperref[app:categorical_sim]{Categorical simulation details}
\item \hyperref[app:maze]{Maze experiment details}
\item \hyperref[app:training]{Experimental details, $\gamma$ sweep, and SynLogic results}
\item \hyperref[app:group_size_results]{Group-size, rare-trajectory, and diversity results}
\item \hyperref[app:grpo_controls]{Regularization and update-scale controls}
\item \hyperref[app:llm_passk_curves]{LLM pass@$k$ tables}
\item \hyperref[app:llm_training_dynamics]{LLM training dynamics}
\item \hyperref[app:statsign]{Statistical significance}
\item \hyperref[app:notation]{Notation}
\end{enumerate}
\endgroup

\section{Additional Related Work}
\label{app:related_work}

\textbf{Distribution Sharpening in RLVR.}
Several studies report that RLVR can improve single-sample accuracy while concentrating the output distribution and reducing performance at large sampling budgets~\citep{dang2025weight, rl_passk_incentivize, wu2025invisible}. \citet{chen2025rethinking} attribute this behavior to overconfidence induced by cross-entropy training and propose confidence limiting. Our analysis focuses on undercoverage of rare correct modes by finite rollout groups and the resulting redistribution toward sampled solutions.

\textbf{Group Size and Sampling Dynamics.}
Rollout count controls both how often group-relative updates are active and which behaviors they expose. \citet{it_takes_2} argue that $N{=}2$ is theoretically justified and compute-efficient, whereas \citet{brorl} show that scaling $N$ ensures non-negative change in total correct mass by improving coverage. Our analysis connects these endpoints: small groups can preserve coverage through inactivity, large groups through broad sampling, and intermediate groups can combine frequent updates with rare-mode misses. Concurrently, \citet{maxrl} derive likelihood-based estimators for correctness-based RL; Focal weighting provides group-level scaling for group-relative estimators.

\textbf{Difficulty-Aware Training.}
Reweighting by difficulty has established roots in Focal loss~\citep{focal} and curriculum learning~\citep{curriculum, curriculum_rlvr}. In RLVR, \citet{daro} dynamically rebalance loss contributions across difficulty groups, while \citet{rewarding_unlikely} identify rank bias in GRPO and up-weight unlikely correct trajectories. Concurrently, \citet{your_gr_biased} analyze difficulty-dependent bias in group-relative advantages and propose history-aware weighting. \citet{diff_smooth_sharp} study selection and reinforcement bias; their final Differential Smoothing method modifies rewards for both correct and incorrect trajectories. F-GRPO uses a fixed transformation of the current group's empirical success rate to scale the entire group-relative update.

\textbf{Entropy and Token-level Approaches.}
The role of entropy in RLVR remains debated, with some advocating maximization for exploration~\citep{entrmech, cheng2025reasoning} and others reporting benefits from minimization~\citep{agarwal2025unreasonable}. Several methods address token-level concentration by reweighting tokens based on entropy dynamics or probability structure~\citep{rethinking, simko, wang2025beyond}. These approaches regulate how probability mass is distributed \emph{within} trajectories; our Focal weighting instead regulates \emph{which prompts} contribute most strongly to group-relative updates.


\section{Group-Relative Objective Details}
\label{app:objective_details}

GRPO optimizes a clipped surrogate objective. Let $o_i = (y_{i,1}, \ldots, y_{i,T_i})$ denote the token sequence for rollout $i$, with importance ratio $r_{i,t}(\theta) = \frac{\pi_\theta(y_{i,t} \mid x, y_{i,<t})}{\pi_{\theta_{\mathrm{old}}}(y_{i,t} \mid x, y_{i,<t})}$. The GRPO objective is

\noindent
\begin{equation}
  \begin{split}
  \mathcal{L}_{\mathrm{GRPO}}(\theta)
  =
  \mathbb{E}_{x}\Bigl[
  \frac{1}{N} \sum_{i=1}^N \frac{1}{T_i} \sum_{t=1}^{T_i} L_{i,t}^{\mathrm{clip}}
  \\
  \qquad
  -\beta\,\mathbb{D}_{\mathrm{KL}}\left(\pi_\theta \,\|\, \pi_{\mathrm{ref}}\right)
  \Bigr]
  \end{split}
  \label{eq:grpo_obj}
\end{equation}

where $L_{i,t}^{\mathrm{clip}} = \min\big(r_{i,t}\widehat{A}_i,\ \mathrm{clip}(r_{i,t}, 1{-}\varepsilon, 1{+}\varepsilon)\widehat{A}_i\big)$.
We set $\beta=0$ following DAPO~\citep{dapo}. DAPO modifies this with asymmetric clipping bounds $\mathrm{clip}(r_{i,t}, 1{-}\varepsilon_{\mathrm{low}}, 1{+}\varepsilon_{\mathrm{high}})$ where $\varepsilon_{\mathrm{high}} > \varepsilon_{\mathrm{low}}$, relaxing the upper bound for low-probability actions.
Our DAPO and F-DAPO runs use this Clip-Higher objective under the same fixed-$N{=}8$ sampling protocol as the other method pairs, without dynamic group filtering.

CISPO~\citep{cispo} clips the importance weights directly rather than the surrogate product. Define the clipped weight

\begin{equation}
\widehat{r}_{i,t} = \mathrm{clip}\big(r_{i,t}, 1{-}\varepsilon_{\mathrm{low}}^{\mathrm{IS}}, 1{+}\varepsilon_{\mathrm{high}}^{\mathrm{IS}}\big),
\label{eq:cispo_rhat}
\end{equation}

and optimizes a REINFORCE-style objective

\begin{equation}
  \begin{split}
  \mathcal{L}_{\mathrm{CISPO}}(\theta)
  &=
  \mathbb{E}_{i,t}\Bigl[
  \mathrm{sg}(\widehat{r}_{i,t})\, \widehat{A}_i^{\mathrm{GRPO}}
  \\
  &\qquad\cdot \log \pi_\theta(y_{i,t} \mid x, y_{i,<t})
  \Bigr]
  \end{split}
  \label{eq:cispo_obj}
\end{equation}
where $\mathrm{sg}(\cdot)$ denotes stop-gradient.

\begin{figure*}[t]
  \centering
  \includegraphics[width=\textwidth]{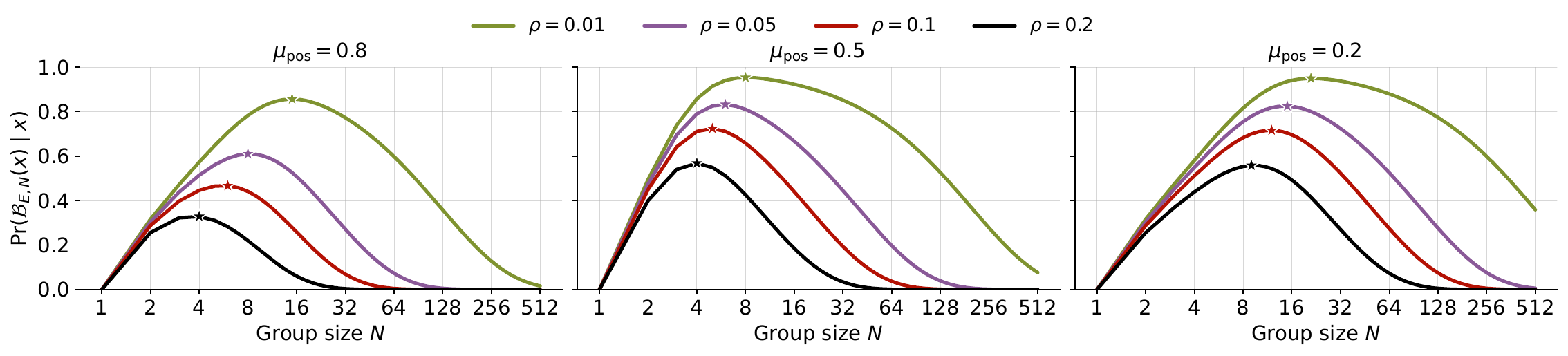}
  \caption{
  Conditional tail-miss probability $\Pr(\mathcal{B}_{E,N}(x)\mid x)$ from Lemma~\ref{lem:P_Btau} versus group size $N$ for rare-mode undercoverage. Each panel fixes $\mu_{\mathrm{pos}}(x) \in \{0.8, 0.5, 0.2\}$; curves vary $\rho = \tau_E(x)/\mu_{\mathrm{pos}}(x)$, the fraction of correct mass in the rare-correct region. Stars mark peaks. For all parameter combinations, the conditional probability peaks at intermediate $N$: small $N$ yields low activity, large $N$ yields good coverage, but moderate $N$ combines active groups with poor coverage of rare modes. Smaller $\rho$ shifts the peak rightward and upward.
  }
  \label{fig:tailmiss_detailed}
\end{figure*}

\section{Categorical Policy Total-Mass Expression}
\label{app:categorical_prelim}

From this update rule, \citet{brorl} derive the one-step change in total correct mass:

\noindent
\begin{equation}
  \begin{aligned}
  \Delta &Q_{\mathrm{pos}}
  = \frac{\eta}{N} \Bigl[
  (R_c - S_R) Q_{\mathrm{neg}} S_{+,2} \\
  &\qquad + (S_R - R_w) Q_{\mathrm{pos}} S_{-,2} \\
  &\qquad + S_R \bigl(
  Q_{\mathrm{pos}} U_{\mathrm{neg},2}
  - Q_{\mathrm{neg}} U_{\mathrm{pos},2}
  \bigr)
  \Bigr].
  \end{aligned}
  \label{eq:delta_Qpos}
\end{equation}

The first two terms are always non-negative: promoting sampled correct actions and demoting sampled incorrect actions both transfer mass to the correct pool. The third term, the unsampled coupling, can be positive or negative depending on $S_R$ and the relative concentration of unsampled masses. As the unsampled second moments decay with $N$, increasing rollout size drives this coupling toward zero.

\section{Proof of Lemma~\ref{lem:P_Btau}}
\label{app:proof_tailmiss}

\begin{proof}
Fix a prompt $x$ and a target subset $E_x\subseteq\mathcal{C}(x)$. Omit $(x)$ for readability and write $\mathcal{C}$ for $\mathcal{C}(x)$ and $E$ for $E_x$. Each rollout falls into one of three disjoint regions:
the target correct subset $E$ with probability $\tau$,
the remaining correct region $\mathcal{C}\setminus E$ with probability $\mu_{\mathrm{pos}}-\tau$,
or the incorrect region $\Omega\setminus \mathcal{C}$ with probability $1-\mu_{\mathrm{pos}}$.

The probability that no rollout lies in $E$ is $(1-\tau)^N$. Conditioned on this event, all rollouts lie in $(\mathcal{C}\setminus E)\cup(\Omega\setminus \mathcal{C})$.
The group is inactive (hence $\mathcal{B}_{E,N}(x)$ does not occur) in two disjoint cases:
all rollouts are correct but not in $E$, with probability $(\mu_{\mathrm{pos}}-\tau)^N$; or all rollouts are incorrect, with probability $(1-\mu_{\mathrm{pos}})^N$.
Thus
\begin{align*}
  \Pr(\mathcal{B}_{E,N}(x)\mid x)
  &= (1-\tau)^N - (\mu_{\mathrm{pos}}-\tau)^N \\
  &\qquad - (1-\mu_{\mathrm{pos}})^N. \qedhere
\end{align*}

\end{proof}

\section{First-order Softmax Expansion and Subset-mass Identity}
\label{app:first_order}

This appendix records standard first-order identities for the softmax map that underlie the analysis in Section~\ref{sec:theory}.

Let $p = \mathrm{softmax}(z)$ over $\mathcal{A}$ and consider a small logit perturbation $\Delta z$. The softmax Jacobian $\frac{\partial p_i}{\partial z_j} = p_i(\mathbf{1}\{i = j\} - p_j)$ implies the first-order probability change
\begin{equation}
\Delta p_i = \sum_{j \in \mathcal{A}} \frac{\partial p_i}{\partial z_j} \Delta z_j = p_i \Big( \Delta z_i - \sum_{j \in \mathcal{A}} p_j \Delta z_j \Big).
\label{eq:app_dp}
\end{equation}

For any subset $\mathcal{S} \subseteq \mathcal{A}$, define its probability mass $Q_{\mathcal{S}} := \sum_{i \in \mathcal{S}} p_i$. Summing~\eqref{eq:app_dp} over $i \in \mathcal{S}$ yields the \emph{subset-mass identity}:
\begin{equation}
\Delta Q_{\mathcal{S}} := \sum_{i \in \mathcal{S}} \Delta p_i = \sum_{i \in \mathcal{S}} p_i \Delta z_i - Q_{\mathcal{S}} \sum_{j \in \mathcal{A}} p_j \Delta z_j.
\label{eq:app_subset_mass}
\end{equation}

The first term captures the direct effect of logit changes on actions in $\mathcal{S}$, while the second term captures the indirect effect through softmax normalization: when probability mass moves elsewhere, $Q_{\mathcal{S}}$ changes even if the logits of actions in $\mathcal{S}$ are unchanged.

\textbf{Application to $\Delta Q_{\mathrm{pos}}$.} Setting $\mathcal{S} = \mathcal{P}$ and using the one-step update~\eqref{eq:dz_update} recovers the mass balance equation~\eqref{eq:delta_Qpos} of \citep{brorl}.

\textbf{Application to $\Delta Q_{\mathrm{u,pos}}$.} Setting $\mathcal{S} = U \cap \mathcal{P}$ (unsampled correct actions) yields Proposition~\ref{prop:delta_upos}. The key observation is that for $i \in U$, we have $R_i = 0$, so $\Delta z_i = -\frac{\eta}{N} S_R p_i$ from~\eqref{eq:dz_update}.

\section{Proof of Proposition~\ref{prop:delta_upos}}
\label{app:proof_upos}

\begin{proof}
Apply the subset-mass identity (Appendix~\ref{app:first_order}, Eq.~\eqref{eq:app_subset_mass}) with $\mathcal{S} = U \cap \mathcal{P}$:
\begin{equation*}
\Delta Q_{\mathrm{u,pos}} = \sum_{i \in U \cap \mathcal{P}} p_i \Delta z_i - Q_{\mathrm{u,pos}} \sum_{j \in \mathcal{A}} p_j \Delta z_j.
\end{equation*}

For $i \in U \cap \mathcal{P}$, we have $R_i = 0$, so by~\eqref{eq:dz_update}, $\Delta z_i = -\frac{\eta}{N} S_R p_i$. Thus the first sum becomes

\begin{equation*}
\begin{aligned}
\sum_{i \in U \cap \mathcal{P}} p_i \Delta z_i
&= -\frac{\eta}{N} S_R
   \sum_{i \in U \cap \mathcal{P}} p_i^2 \\
&= -\frac{\eta}{N} S_R\, U_{\mathrm{pos},2}.
\end{aligned}
\end{equation*}

For the normalization term, partitioning by reward value:
\begin{align*}
&\sum_{j \in \mathcal{A}} p_j \Delta z_j 
= \frac{\eta}{N} \sum_{j \in \mathcal{A}} p_j^2 (R_j - S_R) \nonumber = \\
&\frac{\eta}{N} \Big[ (R_c - S_R) S_{+,2} + (R_w - S_R) S_{-,2} - S_R U_2 \Big],
\end{align*}

where we used $R_j = R_c$ for $j \in S_+$, $R_j = R_w$ for $j \in S_-$, and $R_j = 0$ for $j \in U$. Substituting both expressions yields~\eqref{eq:delta_Qupos}.
\end{proof}

\section{Detailed Term Analysis for Proposition~\ref{prop:delta_upos}}
\label{app:term_analysis}

We analyze each term in~\eqref{eq:delta_Qupos} to understand when unsampled-correct mass decreases.

\textbf{Direct drift term.} The term $-S_R\, U_{\mathrm{pos},2}$ arises because unsampled actions receive zero reward but are still affected by the baseline subtraction. When $S_R > 0$ (reward-positive batch), this term is negative and pushes unsampled-correct mass downward. The magnitude scales with $U_{\mathrm{pos},2}$, the concentration of unsampled-correct probability.

\textbf{Normalization coupling.} The second term couples $Q_{\mathrm{u,pos}}$ to the mass changes elsewhere. The factor in parentheses has three components:
\begin{itemize}[leftmargin=*, itemsep=3pt, topsep=3pt]
\item $(R_c - S_R) S_{+,2} \geq 0$: sampled-correct actions gain probability, which through normalization draws mass away from unsampled-correct actions.
\item $(R_w - S_R) S_{-,2} \leq 0$: sampled-incorrect actions lose probability, which through normalization donates mass to all other actions including unsampled-correct ones.
\item $-S_R U_2$: when $S_R > 0$, unsampled actions (both correct and incorrect) lose probability through baseline subtraction.
\end{itemize}

\textbf{When does $\Delta Q_{\mathrm{u,pos}} < 0$ while $\Delta Q_{\mathrm{pos}} > 0$?} Consider a reward-positive batch ($S_R > 0$) on a prompt with high success probability. In this regime:
\begin{itemize}[leftmargin=*, itemsep=3pt, topsep=3pt]
\item The direct drift $-S_R U_{\mathrm{pos},2} < 0$ actively pushes unsampled-correct mass down.
\item The normalization coupling is dominated by $(R_c - S_R) S_{+,2} > 0$ when sampled-correct mass is concentrated, further draining unsampled-correct mass.
\item Meanwhile, $\Delta Q_{\mathrm{pos}}$ from~\eqref{eq:delta_Qpos} remains positive because its first two terms (mass transfer from incorrect to correct pool) outweigh the unsampled coupling.
\end{itemize}
Thus RLVR can increase total correct mass while concentrating it onto the sampled-correct subset, shrinking the probability of correct actions that happen not to be sampled.

\section{Monotonicity of Sampled Distinct Mass Conditioned on $X$}
\label{app:proof_sr_monotone}

This appendix formalizes the monotonicity claim used in Section~\ref{sec:method} (Focal Weight): as the observed correct count \(X\) increases, the expected distinct sampled-correct mass is non-decreasing. As a corollary, the corresponding categorical baseline \(S_R\) is also monotone in \(X\).

\paragraph{Setup.}
Fix a prompt \(x\) and write \(\pi(o) := \pi_\theta(o\mid x)\) for brevity. Let \(\Omega_x\) be the rollout space and \(\mathcal{C} := \mathcal{C}(x) \subseteq \Omega_x\) the set of correct rollouts (Section~\ref{sec:prelim:rlvr}). Sample \(N\) i.i.d.\ rollouts \(o_1,\ldots,o_N \sim \pi(\cdot)\), and let \(X := \sum_{i=1}^N \mathbb{I}[o_i \in \mathcal{C}]\) be the number of correct rollouts.

Define the \emph{distinct sampled sets}
\[
S_+ := \{\, o_i : o_i \in \mathcal{C}\,\}, \qquad
S_- := \{\, o_i : o_i \notin \mathcal{C}\,\},
\]
where braces denote a set (duplicates removed). Define the corresponding sampled masses
\[
P_{\mathrm{pos}} := \sum_{o \in S_+} \pi(o), \qquad
P_{\mathrm{neg}} := \sum_{o \in S_-} \pi(o).
\]
These are the trajectory-level analogues of the categorical quantities in Section~\ref{sec:prelim:categorical}. As in that section, define
\begin{equation}
S_R := R_c\, P_{\mathrm{pos}} + R_w\, P_{\mathrm{neg}}.
\label{eq:app_SR_def_traj}
\end{equation}

\paragraph{Conditional Distributions.}
Let \(\mu_{\mathrm{pos}} := \Pr_{o\sim \pi}[o\in \mathcal{C}]\). For \(o \in \mathcal{C}\), define the conditional (restricted) distribution
\[
q_{\mathrm{pos}}(o) := \Pr[o_i = o \mid o_i \in \mathcal{C}] = \frac{\pi(o)}{\mu_{\mathrm{pos}}}.
\]
Similarly, for \(o \notin \mathcal{C}\), define
\[
q_{\mathrm{neg}}(o) := \Pr[o_i = o \mid o_i \notin \mathcal{C}] = \frac{\pi(o)}{1-\mu_{\mathrm{pos}}}.
\]
By exchangeability of i.i.d.\ sampling, conditioning on \(X=k\) implies that the \(k\) correct rollouts are i.i.d.\ from \(q_{\mathrm{pos}}\) over \(\mathcal{C}\), and the \(N-k\) incorrect rollouts are i.i.d.\ from \(q_{\mathrm{neg}}\) over \(\Omega_x \setminus \mathcal{C}\).

\begin{lemma}
\label{lem:app_Ppos_Pneg_monotone}
For all integers \(k \in \{0,1,\ldots,N\}\),
\begin{align}
  \mathbb{E}\!\left[P_{\mathrm{pos}} \mid X=k\right]
  &= \sum_{o \in \mathcal{C}} \pi(o)
  \notag\\[-0.3em]
  &\quad{}\times
  \Bigl(1 - (1-q_{\mathrm{pos}}(o))^{k}\Bigr),
  \label{eq:app_E_Ppos_given_X}
  \\[0.4em]
  \mathbb{E}\!\left[P_{\mathrm{neg}} \mid X=k\right]
  &= \sum_{o \notin \mathcal{C}} \pi(o)
  \notag\\[-0.3em]
  &\quad{}\times
  \Bigl(1 - (1-q_{\mathrm{neg}}(o))^{N-k}\Bigr).
  \label{eq:app_E_Pneg_given_X}
\end{align}
Moreover, \(\mathbb{E}[P_{\mathrm{pos}} \mid X=k]\) is non-decreasing in \(k\), and \(\mathbb{E}[P_{\mathrm{neg}} \mid X=k]\) is non-increasing in \(k\).
\end{lemma}

\begin{proof}
We prove the statement for \(P_{\mathrm{pos}}\); the argument for \(P_{\mathrm{neg}}\) is identical with \(N-k\) in place of \(k\).

Condition on \(X=k\). For any fixed \(o \in \mathcal{C}\), the event \(\{o \in S_+\}\) is exactly the event that \(o\) appears at least once among the \(k\) correct i.i.d.\ draws from \(q_{\mathrm{pos}}\). Thus
\[
\Pr(o \in S_+ \mid X=k) = 1 - (1-q_{\mathrm{pos}}(o))^{k}.
\]
Using linearity of expectation and the definition \(P_{\mathrm{pos}}=\sum_{o\in\mathcal{C}} \pi(o)\,\mathbb{I}\{o\in S_+\}\),
\begin{align*}
\mathbb{E}[P_{\mathrm{pos}} \mid X=k] = \sum_{o \in \mathcal{C}} \pi(o)\, \Pr(o \in S_+ \mid X=k)
=\\
\sum_{o \in \mathcal{C}} \pi(o)\,\Big(1 - (1-q_{\mathrm{pos}}(o))^{k}\Big),
\end{align*}
which is~\eqref{eq:app_E_Ppos_given_X}.

To show monotonicity, compute the discrete difference:
\begin{align*}
\mathbb{E}[P_{\mathrm{pos}} \mid X=k{+}1] - \mathbb{E}[P_{\mathrm{pos}} \mid X=k]
=\\
\sum_{o \in \mathcal{C}} \pi(o)\, (1-q_{\mathrm{pos}}(o))^{k}\, q_{\mathrm{pos}}(o)
\;\ge\; 0,
\end{align*}
so \(\mathbb{E}[P_{\mathrm{pos}} \mid X=k]\) is non-decreasing in \(k\).

For \(P_{\mathrm{neg}}\), conditioned on \(X=k\), each \(o\notin\mathcal{C}\) is included in \(S_-\) with probability \(1-(1-q_{\mathrm{neg}}(o))^{N-k}\). This yields~\eqref{eq:app_E_Pneg_given_X}. Since \(N-k\) decreases as \(k\) increases and \(m \mapsto 1-(1-q)^m\) is non-decreasing in \(m\), it follows that \(\mathbb{E}[P_{\mathrm{neg}} \mid X=k]\) is non-increasing in \(k\).
\end{proof}

\begin{corollary}
\label{cor:app_SR_monotone}
Assume standard RLVR rewards \(R_c > R_w\) and \(R_w \le 0\) (Section~\ref{sec:prelim:rlvr}). Then \(\mathbb{E}[S_R \mid X=k]\) is non-decreasing in \(k\).
\end{corollary}

\begin{proof}
By definition~\eqref{eq:app_SR_def_traj} and linearity of expectation,
\begin{align*}
\mathbb{E}[S_R \mid X=k] = R_c\,\mathbb{E}[P_{\mathrm{pos}} \mid X=k] \\
+ R_w\,\mathbb{E}[P_{\mathrm{neg}} \mid X=k].
\end{align*}
By Lemma~\ref{lem:app_Ppos_Pneg_monotone}, the first term is non-decreasing in \(k\) because \(R_c>0\), and the second term is also non-decreasing in \(k\) because \(R_w \le 0\) and \(\mathbb{E}[P_{\mathrm{neg}} \mid X=k]\) is non-increasing in \(k\). Hence their sum is non-decreasing in \(k\).
\end{proof}

\section{Focal Weight Visualization}
\label{app:focal_visualization}

Figure~\ref{fig:advantage_magnitude} visualizes the effect of Focal weighting. With binary rewards, the GRPO advantage magnitudes vary with $\mu_{\mathrm{pos}}(x)$. The Focal weight $g(x) = (1 - \widehat{\mu}_{\mathrm{pos}}(x))^\gamma$ scales these magnitudes, attenuating high-success groups more strongly than mixed low-success ones.

\begin{figure}[H]
\includegraphics[width=0.44\textwidth]{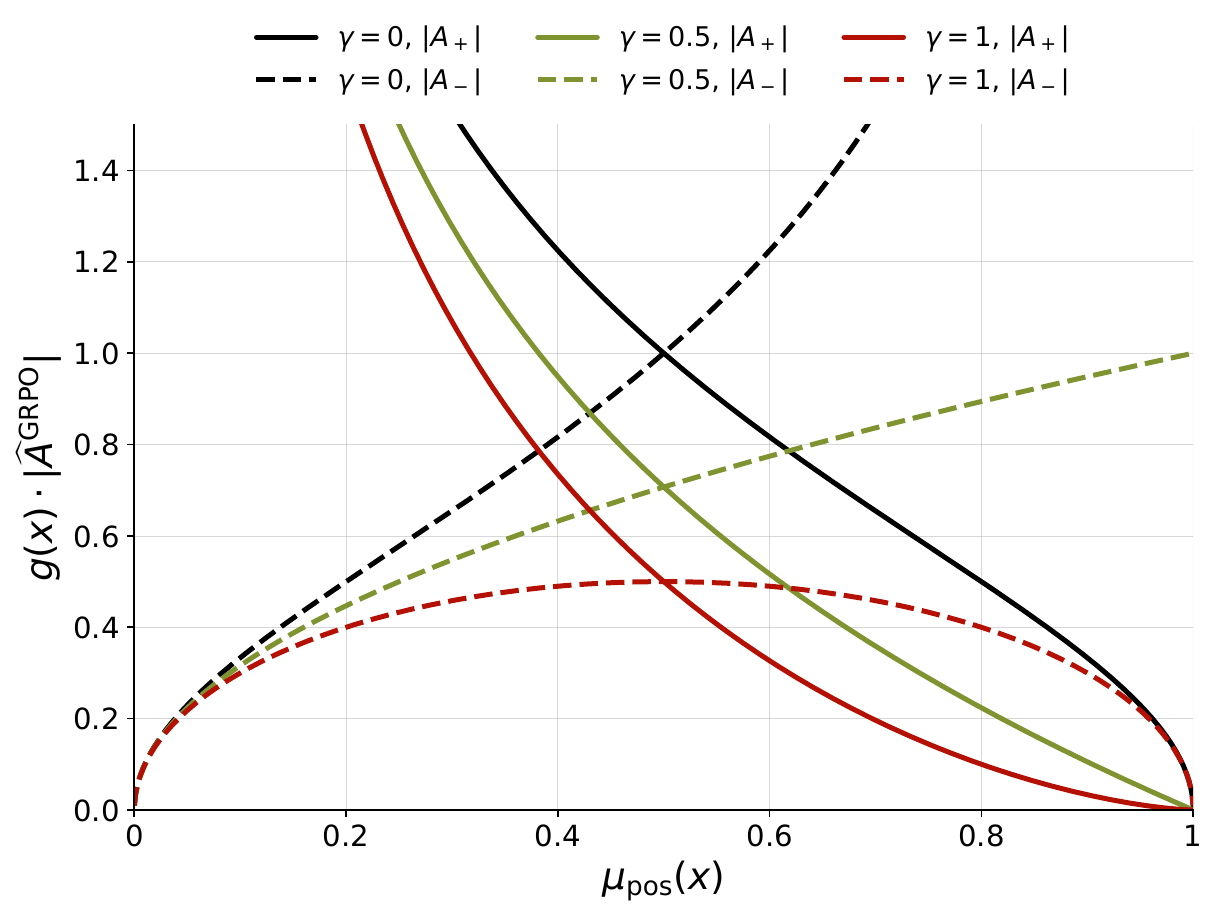}
\caption{
Scaled advantage magnitude $g(x) \cdot |\widehat{A}^{\mathrm{GRPO}}|$ versus success probability $\mu_{\mathrm{pos}}(x)$ for binary rewards. Solid lines: correct rollouts; dashed lines: incorrect rollouts. Higher \(\gamma\) attenuates high empirical-success groups more strongly, thereby changing the relative weighting of active updates.
}
\label{fig:advantage_magnitude}
\end{figure}

\section{Categorical Simulation Details}
\label{app:categorical_sim}

We validate the theoretical framework using a categorical policy simulation. To enable direct comparison with prior work, we adopt the setup of \citet{brorl} with one modification to the learning rate, as described below.

The policy is a softmax distribution over $|\mathcal{A}|=128{,}000$ actions. A subset $\mathcal{A}^+$ of $10{,}000$ actions is designated as correct with reward $R=+1$; the remaining $118{,}000$ actions receive $R=-1$. Following \citet{brorl}, logits are initialized as: one ``anchor'' correct action receives $z_{\mathrm{anchor}}=5.0$; all other correct actions receive $z=3.0$; incorrect actions receive $z=0.0$. Under softmax temperature $1$, this yields initial total correct mass $Q_{\mathrm{pos}} \approx 0.63$, anchor probability $p_{\mathrm{anchor}} \approx 4.7 \times 10^{-4}$, and probability $\tau_{\mathrm{leaf}} \approx 6.3 \times 10^{-5}$ for each non-anchor correct action.

Given this initial distribution, we can compute the conditional tail-miss probability from Lemma~\ref{lem:P_Btau} for a typical non-anchor correct action with $\tau_E = \tau_{\mathrm{leaf}} \approx 6.3 \times 10^{-5}$. Figure~\ref{fig:btau_vs_n} shows $\Pr(\mathcal{B}_{E,N}(x)\mid x)$ as a function of group size $N$. The probability rises steeply for small $N$, plateaus near $1$ for intermediate values, and only declines toward zero for $N \gtrsim 2^{15}$. At $N = 2^{17} = 131{,}072$, $\Pr(\mathcal{B}_{E,N}(x)\mid x) < 10^{-3}$, predicting that such a group size should preserve probability mass on non-anchor correct actions. This prediction aligns with the simulation results in Figure~\ref{fig:categorical_sim}: $N{=}131{,}072$ is the only configuration that maintains $\mathcal{M}_{\mathrm{ret}} \approx 1$ throughout training.

\begin{figure}[H]
\centering
\includegraphics[width=0.5\textwidth]{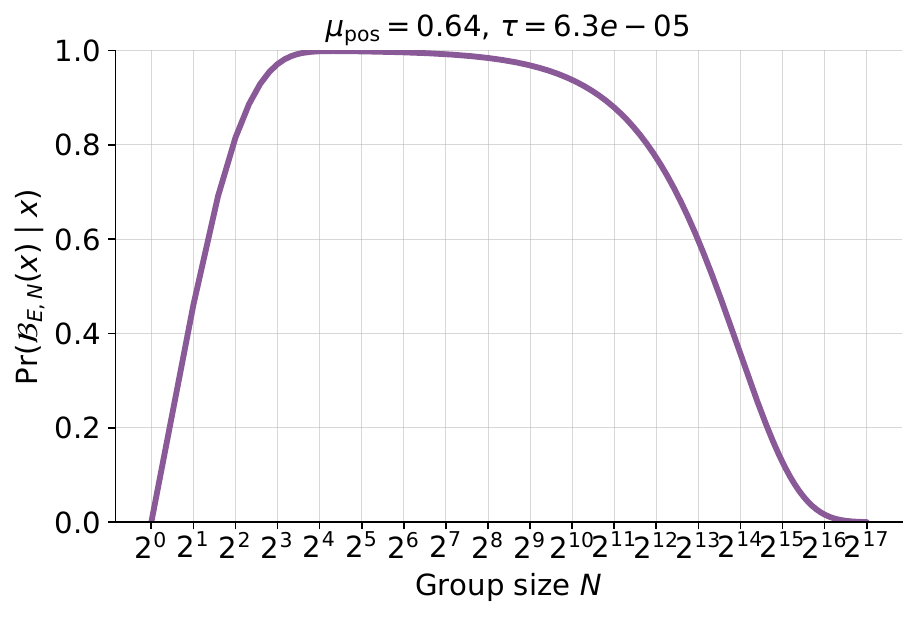}
\caption{Pointwise tail-miss probability $\Pr(\mathcal{B}_{E,N}(x)\mid x)$ versus group size $N$ for $\mu_{\mathrm{pos}} = 0.64$ and $\tau_E = 6.3 \times 10^{-5}$, corresponding to a non-anchor correct action in the simulation. The non-monotonic shape explains the concentration zone: intermediate $N$ maximizes the probability that a correct action is unsampled while the batch contains mixed rewards.}
\label{fig:btau_vs_n}
\end{figure}

At each training step, we sample $N$ actions i.i.d.\ from the current policy, compute group-relative advantages $\tilde{r}_j = R_j - \frac{1}{N}\sum_{k} R_k$, and update logits via gradient ascent on $\mathcal{L} = \frac{1}{N}\sum_j \tilde{r}_j p_j$. When Focal weighting is applied, the objective is scaled by $g = (1 - \widehat{\mu}_{\mathrm{pos}})^\gamma$. We use learning rate $\eta = 10^{-2}$, which differs from $\eta = 10^{-3}$ in \citet{brorl}. At the lower learning rate, policy entropy after $1{,}000$ steps remained above $4$ even for $N{=}65{,}536$, whereas LLM generation entropy during RLVR training is typically below $1$. At the original learning rate, the same qualitative concentration pattern, again mitigated by Focal weighting, appears over $50{,}000$ steps.

We sweep $N \in \{2, 4, \ldots, 131{,}072\}$ and $\gamma \in \{0, 1\}$, running $T{=}1{,}000$ steps per configuration. Results are averaged over $4$ random seeds.

\textbf{Metrics.} We track total correct mass $Q_{\mathrm{pos}}(t) = \sum_{a \in \mathcal{A}^+} \pi_t(a)$ and retained positive mass:
\begin{equation}
\mathcal{M}_{\mathrm{ret}}(t) = 1 - \frac{\sum_{a \in \mathcal{A}^+} \max\bigl(0,\, \pi_0(a) - \pi_t(a)\bigr)}{\sum_{a \in \mathcal{A}^+} \pi_0(a)}.
\label{eq:retained_mass}
\end{equation}
$\mathcal{M}_{\mathrm{ret}}{=}1$ indicates no correct action has lost mass; $\mathcal{M}_{\mathrm{ret}}{\approx}0$ indicates concentration onto a smaller subset.

\begin{figure*}[h!]
  \centering
  \includegraphics[width=\textwidth]{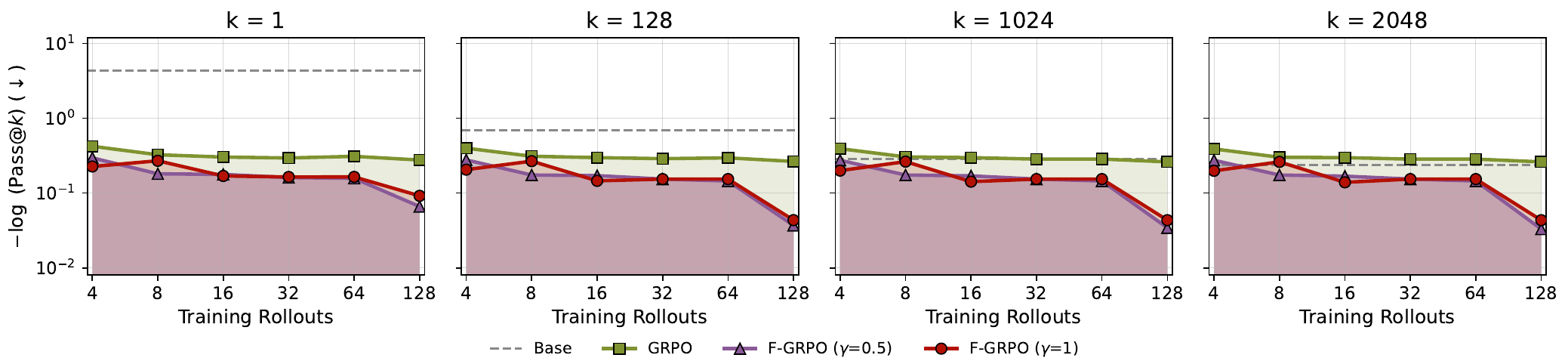}
  \caption{
  Single-solution Maze evaluation across group sizes. Each maze prompt has one target action sequence. The figure reports $-\log(\mathrm{pass@}K)$ for $K\in\{1,128,1024,2048\}$; lower is better. F-GRPO improves target-sequence coverage across group sizes.
  }
  \vspace{-0.5em}
  \label{fig:maze_passk_vs_n}
\end{figure*}

\begin{table*}[h!]
  \centering
  \small
  \setlength{\tabcolsep}{4.0pt}
  \begin{tabular}{c|ccc|ccc|ccc}
  \toprule
  \multirow{2}{*}{\textbf{$N$}} &
  \multicolumn{3}{c|}{\textbf{pass@1}} &
  \multicolumn{3}{c|}{\textbf{pass@128}} &
  \multicolumn{3}{c}{\textbf{pass@256}} \\
  \cmidrule(lr){2-4}\cmidrule(lr){5-7}\cmidrule(lr){8-10}
  & \textbf{GRPO} & \textbf{$\gamma{=}0.5$} & \textbf{$\gamma{=}1.0$}
  & \textbf{GRPO} & \textbf{$\gamma{=}0.5$} & \textbf{$\gamma{=}1.0$}
  & \textbf{GRPO} & \textbf{$\gamma{=}0.5$} & \textbf{$\gamma{=}1.0$} \\
  \midrule
  4 & 65.6 & 74.4 & \textbf{79.7} & 67.1 & 75.8 & \textbf{81.4} & 67.3 & 75.9 & \textbf{81.5} \\
  8 & 72.2 & \textbf{83.4} & 76.3 & 73.3 & \textbf{84.0} & 76.6 & 73.4 & \textbf{84.0} & 76.6 \\
  16 & 73.9 & 83.7 & \textbf{84.4} & 74.2 & 84.3 & \textbf{86.4} & 74.2 & 84.3 & \textbf{86.5} \\
  32 & 74.5 & \textbf{85.1} & 84.9 & 74.9 & \textbf{85.8} & \textbf{85.8} & 75.1 & \textbf{85.8} & \textbf{85.8} \\
  64 & 73.4 & \textbf{85.4} & 84.8 & 74.4 & \textbf{86.5} & 85.8 & 74.7 & \textbf{86.5} & 85.8 \\
  128 & 75.8 & \textbf{93.6} & 91.2 & 76.7 & \textbf{96.4} & 95.7 & 76.9 & \textbf{96.5} & 95.7 \\
  \bottomrule
  \end{tabular}
  \vspace{-0.5em}
  \caption{Final Maze evaluation metrics after the 20K-step run. Values are percentages on the held-out maze set. \(\gamma\) columns denote F-GRPO. \textbf{Bold}: best among GRPO, F-GRPO \(\gamma{=}0.5\), and F-GRPO \(\gamma{=}1.0\) for the same \(N\) and pass@$k$.}
  \label{tab:maze_final_passk}
\end{table*}

\begin{figure*}[h!]
  \centering
  \includegraphics[width=\textwidth]{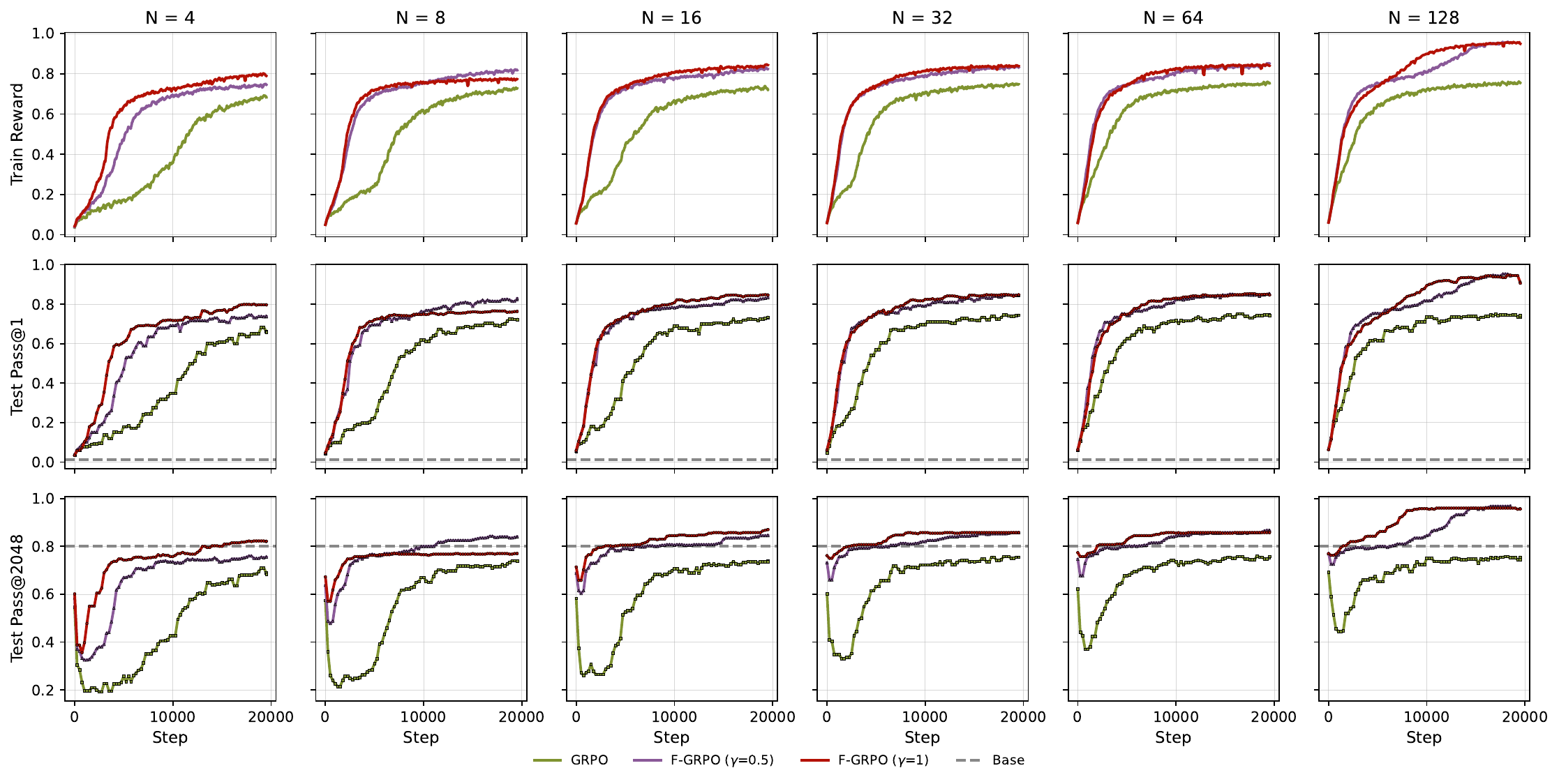}
  \caption{
  Training dynamics on single-solution Maze. GRPO often shows an early drop in evaluation pass@$k$, indicating reduced target-sequence coverage, followed by later recovery. F-GRPO reduces the early degradation across group sizes while improving convergence. Large-$k$ pass@$k$ is used here as a diagnostic for target-sequence coverage.
  }
  \label{fig:maze_training_dynamics}
  \vspace{-1.0em}
\end{figure*}

\section{Maze Experiment Details}
\label{app:maze}

We adapt the maze-navigation setup from concurrent work~\citep{maxrl}, using the shortest-path action sequence as the target for RL and held-out evaluation. The environment uses $1$M procedurally generated $17{\times}17$ training mazes and 256 held-out mazes. Mazes are represented with symbolic grid tokens, and the model autoregressively emits navigation actions ending with a termination token. We use the reported lightweight decoder-only Transformer configuration, with approximately 3M parameters.

Training follows the concurrent group-size sweep. We first run SFT from scratch on provided target trajectories for 1,500 steps with batch size 32, AdamW \cite{adamw}, and learning rate $5{\times}10^{-4}$, which initializes the output format and yields nonzero pass rate. We then run fully on-policy RL with one parameter update per RL step, data batch size 256, learning rate $1{\times}10^{-4}$, and rollout number $N\in\{4,8,16,32,64,128\}$. Final Maze results are reported after 20K RL steps. Each prompt is evaluated against a single target action sequence, so pass@$k$ estimates the probability that at least one of $k$ samples reaches that sequence.

Figure~\ref{fig:maze_passk_vs_n} shows that F-GRPO improves final pass@1 and pass@$k$, and reduces the early training-time drop in pass@$k$. At final evaluation, F-GRPO with $\gamma{=}0.5$ improves pass@1 for every $N$ (74.4--93.6 vs.\ 65.6--75.8 for GRPO) and reduces $-\log(\mathrm{pass@}K)$ across larger $k$; exact final pass@1/128/256 values are in Table~\ref{tab:maze_final_passk}, and the training-dynamics figure additionally reports Test pass@2048. Training dynamics in Figure~\ref{fig:maze_training_dynamics} show that GRPO often exhibits an early evaluation pass@$k$ drop followed by recovery, while F-GRPO reduces this degradation. Because each prompt has a unique target sequence, that early large-$k$ drop cannot be attributed to redistribution among multiple correct trajectories within the same prompt. The Maze results therefore indicate that the observed degradation is not solely a within-prompt multi-solution phenomenon.

Figure~\ref{fig:maze_training_dynamics} reports training reward, pass@1, and diagnostic large-$k$ pass@$k$ across group sizes. The transient evaluation drop under GRPO is consistent with the possibility that finite-group updates selected from training prompts can reduce target-sequence coverage. This dynamics-based evidence is complementary to the final pass@$k$ versus group-size summary in Figure~\ref{fig:maze_passk_vs_n}.

\section{Experimental Details}
\label{app:training}

\subsection{Dataset Preprocessing}
The math experiments use the DeepScaleR math dataset~\citep{deepscaler}. We filter samples longer than 1024 tokens and remove duplicates with conflicting answers, retaining 39,202 samples. The system prompt \texttt{"Please reason step by step, and put your final answer within \textbackslash boxed\{\}."} is prepended to these training inputs. Qwen3-1.7B is trained on SynLogic~\citep{synlogic}. We enable thinking mode during training for both Qwen3 models.

\subsection{Training Configuration}
Training uses the verl pipeline~\citep{verl} with sglang~\citep{sglang} for rollout generation, on 16 NVIDIA H100 GPUs with FSDP2~\citep{fsdp}. Maximum response lengths are 4096 tokens for the Qwen3 models and 8192 tokens for Qwen2.5-7B and Llama-3.2-3B-Instruct. Following~\citep{dapo}, we drop the KL-divergence regularization term and use token-mean loss aggregation. Unless otherwise stated, all experiments use learning rate $1 \times 10^{-6}$ according to \citep{grpo, dapo}.

\textbf{Clipping parameters:} $\epsilon_{\text{low}}{=}0.2$, $\epsilon_{\text{high}}{=}0.2$ for GRPO; $\epsilon_{\text{low}}{=}0.2$, $\epsilon_{\text{high}}{=}0.28$ for DAPO; $\epsilon_{\text{low}}{=}1.0$, $\epsilon_{\text{high}}{=}5.0$ for CISPO, following~\citep{artscaling}. Rewards are 1.0 for correct and 0.0 for incorrect; math rewards are computed with math-verify~\citep{math-verify}. Complete hyperparameters are in Table~\ref{tab:hyperparameters}.

\begin{table}[H]
  \centering
  \small
  \begin{tabular}{lc}
  \toprule
    \textbf{Parameter} & \textbf{Value} \\
    \midrule
    Optimizer & AdamW \\
    $(\beta_1, \beta_2)$ & (0.9, 0.999) \\
    Weight decay & 0.01 \\
    Gradient norm clipping & 1.0 \\
    Learning rate & $1 \times 10^{-6}$ \\
    LR scheduler & Constant \\
    Warmup steps & 15 \\
    Global batch size & 256 \\
    Mini-batch size & 64 \\
    Num training epochs & 10 \\
    PPO epochs & 1 \\
    Sampling temperature & 1.0 \\
    (top-p, top-k) & (1.0, -1) \\
    \bottomrule
  \end{tabular}
  \vspace{-0.5em}
  \caption{Training hyperparameters.}
  \label{tab:hyperparameters}
\end{table}

\textbf{Entropy and KL Regularization:} for the regularization controls in Section~\ref{sec:grpo_controls}, we tune the entropy bonus coefficient over $\{0.0001, 0.001\}$ and the KL penalty coefficient over $\{0.001, 0.01\}$. We select the best checkpoint for each configuration based on average math pass@1. The best-performing coefficients are $0.001$ for both entropy bonus and KL penalty.

\subsection{Evaluation Protocol}

We report unbiased pass@$k$ estimator \cite{pass_at_k}, the probability that at least one of $k$ samples is correct:
\begin{equation}
\text{pass@}k := \mathbb{E}_{\substack{\text{Problems}}}\left[\,1 - \frac{\binom{n-c}{k}}{\binom{n}{k}}\,\right],
\label{eq:pass_at_k}
\end{equation}
where $n$ is the total number of samples and $c$ is the number of correct samples.

For checkpoint selection, we save a checkpoint at the end of each epoch. We choose the best baseline checkpoint by average math pass@1, then compare to the best F-GRPO checkpoint obtained with equal or less compute. Evaluation uses sglang~\citep{sglang}; math benchmarks are scored with math-verify~\citep{math-verify}. Configurations and evaluation prompt settings are in Tables~\ref{tab:eval_config} and~\ref{tab:eval_system_prompts}.


\begin{table}[H]
  \centering
  \small
  \setlength{\tabcolsep}{3pt}
  \renewcommand{\arraystretch}{0.95}
  
  \begin{tabular}{@{}lccc@{}}
  \toprule
  \textbf{Parameter}
  & \makecell{\textbf{Qwen}\\\textbf{2.5-7B}}
  & \makecell{\textbf{Qwen3}\\\textbf{4B / 1.7B}}
  & \makecell{\textbf{Llama3.2}\\\textbf{3B}} \\
  \midrule
  Temperature & 1.0 & 1.0 & 1.0 \\
  top-p       & 1.0 & 1.0 & 1.0 \\
  top-k       & -1  & -1  & -1  \\
  Max length  & 8192 & 4096 & 8192 \\
  \bottomrule
  \end{tabular}
  
  \vspace{-0.5em}
  \caption{Evaluation configurations.}
  \label{tab:eval_config}
\end{table}

\begin{table*}[t!]
\centering
\scriptsize
\resizebox{\textwidth}{!}{%
\begin{tabular}{lp{4.0cm}p{4.7cm}p{4.7cm}}
\toprule
\textbf{Benchmark} & \textbf{Qwen2.5-7B} & \makecell{\textbf{Qwen3-4B-Base}\\\textbf{Qwen3-1.7B}} & \textbf{Llama-3.2-3B-Instruct} \\
\midrule
Mathematical reasoning & \texttt{Please reason step by step, and put your final answer within \textbackslash boxed\{\}.} & \texttt{System: Please reason step by step, and put your final answer within \textbackslash boxed\{\}. Thinking: on.} & \texttt{Cutting Knowledge Date: December 2023\textbackslash nToday Date: [date]\textbackslash nPlease reason step by step, and put your final answer within \textbackslash boxed\{\}.} \\
\midrule
GPQA Diamond & \texttt{Please reason step by step, and put your final answer within \textbackslash boxed\{\}.} & \texttt{System: Please reason step by step, and put your final answer within \textbackslash boxed\{\}. Thinking: on.} & \texttt{Cutting Knowledge Date: December 2023\textbackslash nToday Date: [date]\textbackslash nPlease reason step by step, and put your final answer within \textbackslash boxed\{\}.} \\
\midrule
IFEval & \texttt{You are a helpful assistant.} & \texttt{No system prompt. Thinking: off.} & \texttt{Cutting Knowledge Date: December 2023\textbackslash nToday Date: [date]} \\
\midrule
SynLogic & \texttt{You are a helpful assistant.} & \texttt{No system prompt. Thinking: on.} & \texttt{Cutting Knowledge Date: December 2023\textbackslash nToday Date: [date]} \\
\bottomrule
\end{tabular}%
}
\vspace{-0.5em}
\caption{Evaluation prompt settings.}
\label{tab:eval_system_prompts}
\end{table*}

\subsection{Focal Weight Hyperparameter $\gamma$}
\label{app:gamma_sweep}

\begin{table*}[h!]
  \centering
  \small
  \begin{tabular}{lccc}
  \toprule
  \textbf{Model} & \textbf{F-GRPO $\gamma$} & \textbf{F-DAPO $\gamma$} & \textbf{F-CISPO $\gamma$} \\
  \midrule
  Qwen2.5-7B & 0.5 & 0.5 & 1.0 \\
  Qwen3-4B-Base & 0.5 & 0.5 & 0.5 \\
  Llama-3.2-3B-Instruct & 0.5 & 1.0 & 0.5 \\
  Qwen3-1.7B (SynLogic) & 0.5 & 1.0 & 1.0 \\
  \bottomrule
  \end{tabular}
  \vspace{-0.5em}
  \caption{Selected Focal weight \(\gamma\) for each method-model setup at $N{=}8$ (Tables~\ref{tab:all_models} and~\ref{tab:synlogic_summary}). The sweep range is \(\{0.5, 1.0, 2.0\}\).}
  \label{tab:gamma_selected}
\end{table*}

\begin{table*}[h!]
  \centering
  \small
  \setlength{\tabcolsep}{2.8pt}
  \begin{tabular}{lc|cc|cc|cc}
  \toprule
  \multirow{2}{*}{\textbf{Method}} & \multirow{2}{*}{\(\boldsymbol{\gamma}\)} &
  \multicolumn{2}{c}{\textbf{Qwen2.5-7B}} &
  \multicolumn{2}{c}{\textbf{Qwen3-4B-Base}} &
  \multicolumn{2}{c}{\textbf{Llama-3.2-3B-Instruct}} \\
  \cmidrule(lr){3-4}\cmidrule(lr){5-6}\cmidrule(lr){7-8}
  & & \textbf{Avg. Math} & \textbf{Avg. OOD} & \textbf{Avg. Math} & \textbf{Avg. OOD} & \textbf{Avg. Math} & \textbf{Avg. OOD} \\
  \midrule
  \rule{0pt}{2.0ex}GRPO & 0 & \utilde{37.3} / 64.1 & 17.1 / 55.9 & 39.9 / 71.1 & 22.2 / 67.9 & \textbf{23.0} / 59.9 & \textbf{25.5} / \utilde{56.5} \\
  GRPO & 0.5 & \textbf{38.6} / 70.3 & \textbf{19.2} / \textbf{63.3} & \textbf{42.7} / \utilde{76.0} & \textbf{24.3} / 73.6 & \textbf{23.0} / \utilde{63.4} & \utilde{25.4} / \textbf{57.6} \\
  GRPO & 1.0 & 37.2 / \utilde{70.4} & 17.9 / 60.6 & 41.2 / 73.5 & 22.5 / \textbf{75.9} & \utilde{22.5} / 63.2 & 23.7 / 54.3 \\
  GRPO & 2.0 & 36.1 / \textbf{70.7} & \utilde{18.9} / \utilde{62.3} & \utilde{41.3} / \textbf{77.8} & \utilde{23.2} / \utilde{73.7} & 21.8 / \textbf{64.0} & 25.0 / 54.6 \\
  \midrule
  DAPO & 0 & 39.4 / 69.3 & 15.7 / 58.4 & \utilde{45.1} / 76.1 & \utilde{20.7} / 74.4 & \utilde{24.3} / 54.2 & 23.9 / 51.3 \\
  DAPO & 0.5 & \textbf{40.5} / 72.5 & \utilde{17.9} / 63.6 & \textbf{45.3} / 76.4 & \textbf{23.2} / \utilde{76.8} & 21.6 / 55.3 & 24.1 / 52.3 \\
  DAPO & 1.0 & \utilde{39.8} / \utilde{74.9} & \textbf{18.0} / \textbf{65.7} & 43.5 / \utilde{79.0} & 18.7 / 76.7 & \textbf{24.8} / \textbf{62.3} & \textbf{24.8} / \utilde{55.4} \\
  DAPO & 2.0 & 39.5 / \textbf{76.8} & 17.6 / \utilde{64.0} & 42.9 / \textbf{80.0} & 18.2 / \textbf{81.3} & 23.9 / \utilde{61.1} & \utilde{24.2} / \textbf{55.5} \\
  \midrule
  CISPO & 0 & \textbf{39.5} / 73.2 & 14.9 / 59.0 & \utilde{45.8} / 75.3 & 22.6 / 76.5 & \utilde{24.1} / 58.0 & \textbf{25.7} / 52.5 \\
  CISPO & 0.5 & \utilde{39.0} / 72.1 & 17.2 / 59.4 & \textbf{46.0} / \utilde{79.2} & \textbf{23.9} / \utilde{80.0} & \textbf{24.5} / 59.7 & \utilde{25.0} / \utilde{53.0} \\
  CISPO & 1.0 & \textbf{39.5} / \textbf{76.8} & \textbf{18.1} / \textbf{65.9} & 43.6 / 79.1 & 21.8 / 76.5 & 23.9 / \utilde{59.8} & 24.3 / 52.9 \\
  CISPO & 2.0 & 38.0 / \utilde{75.1} & \utilde{17.9} / \utilde{64.4} & 44.2 / \textbf{79.5} & \utilde{22.9} / \textbf{80.7} & 22.0 / \textbf{63.8} & 23.5 / \textbf{57.0} \\
  \bottomrule
  \end{tabular}
  \vspace{0.5em}
  \caption{Focal-weight sensitivity at $N{=}8$ across models and base methods. Metrics are average pass@1 / pass@256. \(\gamma=0\) denotes the unweighted baseline for each base method. \textbf{Bold}: best; $\dagger$: second best within each base method and metric column.}
  \label{tab:gamma_sweep_full}
\end{table*}

Following the model-selection protocol above, we sweep the Focal exponent $\gamma \in \{0.5, 1.0, 2.0\}$ for each Focal-weighted method (F-GRPO, F-DAPO, F-CISPO). For reproducibility, the selected \(\gamma\) values for the setups reported in Tables~\ref{tab:all_models} and~\ref{tab:synlogic_summary} are summarized in Table~\ref{tab:gamma_selected}. Tables~\ref{tab:gamma_sweep_full} and~\ref{tab:synlogic_full} make the resulting pass@1 / pass@256 trade-off explicit, with \(\gamma=0\) denoting the corresponding unweighted baseline. Across model-method combinations, the selected values in Table~\ref{tab:gamma_selected} are always \(\gamma=0.5\) or \(\gamma=1.0\). Larger \(\gamma=2.0\) sometimes further improves pass@256 but more often introduces pass@1 trade-offs, suggesting that the practical search range is small.

\subsection{SynLogic Results}

Table~\ref{tab:synlogic_full} reports the complete Qwen3-1.7B results for the SynLogic setup across all evaluated values of \(\gamma\).

\begin{table*}[h!]
\centering
\scriptsize
\setlength{\tabcolsep}{2.2pt}
\resizebox{\textwidth}{!}{%
\begin{tabular}{l|c|c|cccccc}
\Xhline{1.1pt}
\rule{0pt}{2.0ex} & \multicolumn{1}{c|}{\textbf{In-domain}} & \multicolumn{7}{c}{\textbf{Out-of-domain}} \\
\hhline{~|-|-|------}
\rule{0pt}{2.0ex}\textbf{Method} & \textbf{SynLogic} & \textbf{Avg. OOD} & \textbf{AIME24} & \textbf{AIME25} & \textbf{AMC23} & \textbf{MATH500} & \textbf{IFEval} & \makecell{\textbf{GPQA}\\\textbf{Diamond}} \\
\Xhline{1.1pt}
\rule{0pt}{2.5ex}GRPO & \utilde{40.1} / 74.8 & 31.7 / 60.6 & 8.2 / 42.2 & 14.0 / 28.9 & 46.2 / 82.3 & 74.1 / 93.0 & 33.8 / 46.6 & 14.0 / 70.3 \\
\rowcolor{gray!15}
F-GRPO ($\gamma{=}0.5$) & \textbf{41.8} / \textbf{79.4} & \utilde{35.8} / \utilde{63.9} & \utilde{13.4} / \textbf{46.3} & \utilde{15.6} / \textbf{31.8} & \utilde{52.1} / \utilde{86.5} & \utilde{78.3} / \textbf{94.6} & 34.1 / 43.5 & \textbf{21.4} / \utilde{80.7} \\
F-GRPO ($\gamma{=}1$) & 37.6 / \utilde{77.3} & \textbf{36.0} / \textbf{64.9} & \textbf{14.0} / \utilde{44.6} & \textbf{16.7} / \utilde{31.5} & \textbf{52.6} / \textbf{87.0} & \textbf{78.4} / \utilde{94.1} & \textbf{35.1} / \textbf{51.3} & \utilde{19.0} / \textbf{81.2} \\
F-GRPO ($\gamma{=}2$) & 37.1 / 76.5 & 32.8 / 61.7 & 10.2 / 41.7 & 14.2 / 30.2 & 47.8 / 83.1 & 74.4 / 93.9 & \utilde{34.3} / \utilde{46.7} & 16.2 / 74.4 \\[0.5ex]
\Xhline{0.2pt}
\rule{0pt}{2.5ex}DAPO & \utilde{57.1} / 86.3 & 38.7 / 67.1 & \textbf{21.0} / 56.3 & 19.7 / 39.3 & 56.8 / \utilde{94.8} & \utilde{82.4} / 95.6 & 33.0 / 45.6 & 19.2 / 70.9 \\
F-DAPO ($\gamma{=}0.5$) & 56.1 / 86.8 & \utilde{39.0} / 67.7 & \utilde{20.7} / 59.7 & \textbf{20.7} / \utilde{39.9} & \utilde{57.4} / 93.1 & \textbf{82.6} / 95.7 & \utilde{33.9} / \textbf{46.6} & 18.8 / 71.1 \\
\rowcolor{gray!15}
F-DAPO ($\gamma{=}1$) & \textbf{58.6} / \utilde{87.2} & \textbf{39.1} / \utilde{69.4} & 19.7 / \utilde{60.0} & \utilde{19.9} / \textbf{40.8} & \textbf{59.1} / 94.2 & \utilde{82.4} / \utilde{96.3} & 33.7 / 46.3 & \utilde{20.2} / \utilde{78.6} \\
F-DAPO ($\gamma{=}2$) & 53.4 / \textbf{87.4} & \utilde{39.0} / \textbf{70.1} & 18.7 / \textbf{60.5} & 18.9 / 37.2 & 56.8 / \textbf{95.0} & 82.0 / \textbf{96.8} & \textbf{34.2} / \utilde{46.4} & \textbf{23.6} / \textbf{84.8} \\[0.5ex]
\Xhline{0.2pt}
\rule{0pt}{2.5ex}CISPO & 45.5 / 78.5 & 36.2 / 63.4 & \textbf{19.0} / 51.1 & 16.0 / \utilde{34.1} & 52.9 / \utilde{89.2} & 79.0 / 93.9 & \utilde{34.2} / 44.8 & 16.1 / 67.5 \\
F-CISPO ($\gamma{=}0.5$) & \utilde{48.7} / \utilde{80.4} & \textbf{39.0} / \utilde{66.7} & \utilde{18.0} / \textbf{54.7} & \textbf{18.8} / \textbf{34.6} & \utilde{55.5} / 89.0 & \utilde{81.1} / \utilde{95.2} & \textbf{35.0} / \textbf{45.7} & \textbf{25.8} / \utilde{80.9} \\
\rowcolor{gray!15}
F-CISPO ($\gamma{=}1$) & \textbf{51.0} / \textbf{85.2} & 37.4 / 64.2 & 16.7 / 48.3 & \utilde{18.5} / 30.5 & 54.7 / \utilde{89.2} & 79.8 / 94.7 & 32.8 / \utilde{44.9} & 21.9 / 77.4 \\
F-CISPO ($\gamma{=}2$) & 38.1 / 78.1 & \utilde{38.7} / \textbf{66.9} & 16.5 / \utilde{54.3} & 17.7 / 33.8 & \textbf{57.9} / \textbf{90.5} & \textbf{81.2} / \textbf{96.2} & 33.5 / 43.8 & \utilde{25.3} / \textbf{83.1} \\[0.5ex]
\Xhline{1.1pt}
\end{tabular}%
}
\caption{Full Qwen3-1.7B SynLogic results at $N{=}8$. Metrics are pass@1 / pass@256. Avg. OOD is the mean over the six OOD benchmarks defined in Section~\ref{sec:exp_setup}. \textbf{Bold}: best; $\dagger$: second best within each base method and metric column; gray rows mark $\gamma$ selected by SynLogic pass@1.}
\label{tab:synlogic_full}
\end{table*}

\section{Group Size Comparison: Full Results}
\label{app:group_size_results}

\subsection{Per-Benchmark Results}

Table~\ref{tab:grpo_n_comparison_table} provides full per-benchmark results for the group size comparison discussed in Section~\ref{sec:results:scaling}.

\begin{table*}[h!]
\centering
\scriptsize
\newcommand{\pmsep}{\hspace{0.45mm}}
\setlength{\tabcolsep}{2.2pt}
\setlength{\extrarowheight}{-0.0pt}
\resizebox{\textwidth}{!}{%
\begin{tabular}{l|c|cccccc|c|ccc}
\Xhline{1.1pt}
\rule{0pt}{2.0ex} & \multicolumn{7}{c|}{\textbf{In-domain}} & \multicolumn{4}{c}{\textbf{Out-of-domain}} \\
\hhline{~|-|------|-|---}
\rule{0pt}{2.0ex} \textbf{Method} & \textbf{Avg.} & \textbf{AIME24} & \textbf{AIME25} & \textbf{AMC} & \textbf{MATH500} & \textbf{Minerva} & \textbf{Olympiad} & \textbf{Avg. OOD} & \textbf{IFEval} & \textbf{SynLogic} & \textbf{GPQA}  \\
\Xhline{1.1pt}
\rule{0pt}{2.6ex}GRPO $N{=}2$ & 36.2\pmsep/\pmsep\textbf{75.0} & 12.7\pmsep/\pmsep\textbf{59.1} & 8.3\pmsep/\pmsep\textbf{56.0} & 51.9\pmsep/\pmsep\textbf{97.0} & 74.5\pmsep/\pmsep\textbf{96.7} & 33.2\pmsep/\pmsep\textbf{65.6} & 36.7\pmsep/\pmsep\textbf{75.7} & 18.0\pmsep/\pmsep\textbf{67.3} & 29.4\pmsep/\pmsep\textbf{77.2} & 6.7\pmsep/\pmsep 54.3 & \textbf{17.8}\pmsep/\pmsep\textbf{70.3} \\
\rule{0pt}{2.0ex}GRPO $N{=}4$ & 36.4\pmsep/\pmsep 71.1 & 9.6\pmsep/\pmsep 48.9 & 5.4\pmsep/\pmsep 52.6 & 53.3\pmsep/\pmsep 95.9 & 76.3\pmsep/\pmsep 94.7 & 35.2\pmsep/\pmsep 61.3 & 38.5\pmsep/\pmsep 73.0 & 18.7\pmsep/\pmsep 59.7 & 32.1\pmsep/\pmsep 72.2 & \textbf{9.3}\pmsep/\pmsep 57.0 & 14.7\pmsep/\pmsep 49.8 \\
\rowcolor{gray!15}
\rule{0pt}{2.0ex}GRPO $N{=}8$ & 37.3\pmsep/\pmsep 64.1 & 15.0\pmsep/\pmsep 37.7 & 6.7\pmsep/\pmsep 40.8 & 52.9\pmsep/\pmsep 87.3 & 75.8\pmsep/\pmsep 92.8 & \textbf{36.0}\pmsep/\pmsep 60.2 & 37.8\pmsep/\pmsep 65.8 & 17.1\pmsep/\pmsep 55.9 & 32.1\pmsep/\pmsep 70.3 & 7.9\pmsep/\pmsep 51.3 & 11.3\pmsep/\pmsep 46.2 \\
\rule{0pt}{2.0ex}GRPO $N{=}16$ & 38.4\pmsep/\pmsep 67.5 & 13.0\pmsep/\pmsep 43.9 & 7.1\pmsep/\pmsep 48.0 & 57.7\pmsep/\pmsep 93.4 & 76.7\pmsep/\pmsep 93.2 & 35.7\pmsep/\pmsep 58.0 & \textbf{40.1}\pmsep/\pmsep 68.4 & 17.7\pmsep/\pmsep 55.7 & 31.6\pmsep/\pmsep 68.9 & 8.6\pmsep/\pmsep 49.6 & 12.9\pmsep/\pmsep 48.6 \\
\rule{0pt}{2.0ex}GRPO $N{=}32$ & \textbf{39.2}\pmsep/\pmsep 70.1 & 13.0\pmsep/\pmsep 50.2 & \textbf{10.4}\pmsep/\pmsep 49.5 & \textbf{60.9}\pmsep/\pmsep 95.5 & \textbf{77.3}\pmsep/\pmsep 94.3 & 34.9\pmsep/\pmsep 59.9 & 38.9\pmsep/\pmsep 71.3 & 17.7\pmsep/\pmsep 61.7 & 31.0\pmsep/\pmsep 71.4 & 8.9\pmsep/\pmsep\textbf{61.6} & 13.4\pmsep/\pmsep 51.9 \\[0.5ex]
\Xhline{0.2pt}
\rowcolor{gray!15}
\rule{0pt}{2.5ex}F-GRPO $N{=}8$ & 38.6\pmsep/\pmsep 70.3 & \textbf{15.9}\pmsep/\pmsep 46.2 & 10.1\pmsep/\pmsep 52.6 & 56.2\pmsep/\pmsep 96.3 & 76.2\pmsep/\pmsep 95.1 & 35.7\pmsep/\pmsep 60.3 & 37.5\pmsep/\pmsep 71.6 & \textbf{19.2}\pmsep/\pmsep 63.3 & \textbf{34.0}\pmsep/\pmsep 75.7 & 8.7\pmsep/\pmsep 57.0 & 15.0\pmsep/\pmsep 57.3 \\[0.5ex]
\Xhline{1.1pt}
\end{tabular}%
}
\vspace{-0.5em}
\caption{GRPO with different $N$ and F-GRPO on both in-domain math and out-of-domain benchmarks (Qwen2.5-7B). Pass@1 / Pass@256. \textbf{Bold}: best. Gray rows mark the same group size $N{=}8$ comparison.}
\label{tab:grpo_n_comparison_table}
\end{table*}

\begin{table*}[t!]
  \centering
  \scriptsize
  \newcommand{\pmsep}{\hspace{0.45mm}}
  \setlength{\tabcolsep}{2.2pt}
  \setlength{\extrarowheight}{-0.0pt}
  \resizebox{\textwidth}{!}{%
  \begin{tabular}{l|c|cccccc|c|ccc}
  \Xhline{1.1pt}
  \rule{0pt}{2.0ex} & \multicolumn{7}{c|}{\textbf{In-domain}} & \multicolumn{4}{c}{\textbf{Out-of-domain}} \\
  \hhline{~|-|------|-|---}
  \rule{0pt}{2.0ex} \textbf{Method} & \textbf{Avg.} & \textbf{AIME24} & \textbf{AIME25} & \textbf{AMC} & \textbf{MATH500} & \textbf{Minerva} & \textbf{Olympiad} & \textbf{Avg. OOD} & \textbf{IFEval} & \textbf{SynLogic} & \textbf{GPQA}  \\
  \Xhline{1.1pt}
  \rule{0pt}{2.0ex}Base & 10.9\pmsep/\pmsep 70.5 & 2.6\pmsep/\pmsep 48.0 & 1.4\pmsep/\pmsep 43.7 & 16.6\pmsep/\pmsep\textbf{97.3} & 26.5\pmsep/\pmsep\utilde{96.5} & 5.8\pmsep/\pmsep 62.1 & 12.4\pmsep/\pmsep\utilde{75.7} & 9.3\pmsep/\pmsep\textbf{70.7} & 21.5\pmsep/\pmsep\utilde{78.8} & 2.3\pmsep/\pmsep 39.7 & 4.2\pmsep/\pmsep\textbf{93.8} \\
  \Xhline{0.2pt}
  \rule{0pt}{2.0ex}GRPO low-LR & \utilde{37.8}\pmsep/\pmsep 69.2 & \utilde{15.0}\pmsep/\pmsep 51.0 & 8.7\pmsep/\pmsep 44.5 & 53.3\pmsep/\pmsep 95.3 & 75.9\pmsep/\pmsep 94.4 & 35.0\pmsep/\pmsep 59.7 & \textbf{38.7}\pmsep/\pmsep 70.4 & 16.4\pmsep/\pmsep 57.9 & 30.0\pmsep/\pmsep 69.3 & 7.8\pmsep/\pmsep 55.4 & 11.5\pmsep/\pmsep 49.1 \\
  \rule{0pt}{2.0ex}GRPO ($\mathcal{H}$) & \utilde{37.8}\pmsep/\pmsep 69.5 & 14.9\pmsep/\pmsep 48.9 & 7.3\pmsep/\pmsep 52.2 & \utilde{55.8}\pmsep/\pmsep 90.8 & 75.6\pmsep/\pmsep 94.6 & 34.9\pmsep/\pmsep 61.3 & \utilde{38.2}\pmsep/\pmsep 69.2 & 18.7\pmsep/\pmsep 59.9 & 32.1\pmsep/\pmsep 71.9 & \textbf{9.8}\pmsep/\pmsep\textbf{59.9} & 14.3\pmsep/\pmsep 47.8 \\
  GRPO (KL) & 37.2\pmsep/\pmsep\utilde{72.0} & 13.2\pmsep/\pmsep\utilde{53.4} & 8.7\pmsep/\pmsep\textbf{53.7} & 52.1\pmsep/\pmsep 95.9 & \textbf{76.7}\pmsep/\pmsep 95.2 & 34.7\pmsep/\pmsep 61.5 & 38.0\pmsep/\pmsep 72.3 & \textbf{19.4}\pmsep/\pmsep 60.0 & 32.4\pmsep/\pmsep 70.8 & \utilde{8.8}\pmsep/\pmsep 51.7 & \textbf{17.1}\pmsep/\pmsep 57.5 \\
  \Xhline{0.2pt}
  \rule{0pt}{2.0ex}RUL & 37.1\pmsep/\pmsep 69.1 & 13.7\pmsep/\pmsep 48.7 & 8.2\pmsep/\pmsep 42.3 & 54.2\pmsep/\pmsep 92.9 & 75.2\pmsep/\pmsep 95.8 & 33.5\pmsep/\pmsep\utilde{62.4} & 37.8\pmsep/\pmsep 72.5 & 18.1\pmsep/\pmsep 61.7 & 30.1\pmsep/\pmsep 72.1 & 8.6\pmsep/\pmsep 56.0 & \utilde{15.5}\pmsep/\pmsep 56.9 \\
  \rule{0pt}{2.0ex}DARO & 37.3\pmsep/\pmsep 69.4 & 13.6\pmsep/\pmsep 45.6 & 7.0\pmsep/\pmsep 46.6 & 55.4\pmsep/\pmsep\utilde{97.2} & 74.7\pmsep/\pmsep 94.7 & \utilde{35.5}\pmsep/\pmsep 62.2 & 37.2\pmsep/\pmsep 70.1 & 18.8\pmsep/\pmsep 63.3 & 32.8\pmsep/\pmsep 75.2 & 8.7\pmsep/\pmsep 56.4 & 14.9\pmsep/\pmsep 58.2 \\
  DS-GRPO & 37.7\pmsep/\pmsep\textbf{73.8} & 14.3\pmsep/\pmsep\textbf{57.7} & \textbf{12.2}\pmsep/\pmsep\utilde{53.6} & 51.6\pmsep/\pmsep 95.5 & \utilde{76.6}\pmsep/\pmsep\textbf{96.8} & 33.5\pmsep/\pmsep\textbf{62.8} & 37.8\pmsep/\pmsep\textbf{76.7} & 17.9\pmsep/\pmsep\utilde{68.3} & \utilde{33.4}\pmsep/\pmsep\textbf{85.9} & 7.5\pmsep/\pmsep\utilde{57.7} & 12.7\pmsep/\pmsep\utilde{61.3} \\
  \rowcolor{gray!15}
  \rule{0pt}{2.0ex}F-GRPO & \textbf{38.6}\pmsep/\pmsep 70.3 & \textbf{15.9}\pmsep/\pmsep 46.2 & \utilde{10.1}\pmsep/\pmsep 52.6 & \textbf{56.2}\pmsep/\pmsep 96.3 & 76.2\pmsep/\pmsep 95.1 & \textbf{35.7}\pmsep/\pmsep 60.3 & 37.5\pmsep/\pmsep 71.6 & \utilde{19.2}\pmsep/\pmsep 63.3 & \textbf{34.0}\pmsep/\pmsep 75.7 & 8.7\pmsep/\pmsep 57.0 & 15.0\pmsep/\pmsep 57.3 \\
  \Xhline{1.1pt}
  \end{tabular}%
  }
  \vspace{-0.5em}
  \caption{Qwen2.5-7B base-model reference and baseline comparison; trained methods use $N{=}8$. GRPO low-LR uses $\mathrm{lr}_{\mathrm{low}}=6.8{\times}10^{-7}$; GRPO-$\mathcal{H}$ and GRPO-KL use coefficient $0.001$. Pass@1 / pass@256. \textbf{Bold}: best; $\dagger$: second best.}
  \label{tab:grpo_controls_table}
\end{table*}

\subsection{NLL on Rare-Correct Trajectories}
\label{app:nll_rare}

To construct a proxy for rare-correct modes, we sample 256 prompts from the training set and generate 800 rollouts per prompt from the base model, retaining only correct trajectories. For each retained trajectory, we compute its length-normalized NLL under the base model. We define the ``rare-correct`` subset as the top 1\% by base-model NLL among these correct trajectories, yielding 1,263 trajectories in total. We then compute the NLL of this fixed subset under each trained model; larger values indicate reduced probability assigned to these initially low-probability correct solutions.

\paragraph{Full-trajectory probability on fixed rare subsets.}
Using the same base-model generations, we retain 97 prompts with at least 100 distinct completed correct trajectories and select the five lowest-likelihood trajectories under the base model for each prompt. At matched $N{=}8$, their summed full-trajectory probability is higher under F-GRPO than GRPO on 89/97 prompts. When normalized by the total probability assigned to the same prompt-specific correct pool, the corresponding share is higher under F-GRPO on 77/97 prompts.

\subsection{Correct-Solution Diversity}
\label{app:solution_diversity}

We compare the selected Qwen2.5-7B GRPO and F-GRPO checkpoints from Table~\ref{tab:all_models}, both trained with $N{=}8$, using the 1,024 evaluation responses per problem. For the embedding-based analysis, we encode every correct response with Qwen3-Embedding-8B~\citep{qwen3_embedding} using a maximum length of 8,192 tokens. Following prior embedding-based diversity analyses~\citep{karouzos2026collapse,dra_grpo}, we retain prompts for which both models produce at least two correct responses, compute the mean pairwise cosine distance within each prompt, and average over prompts. This yields 11, 13, 36, 461, 159, and 434 common prompts for AIME24, AIME25, AMC23, MATH500, Minerva, and Olympiad, respectively.

Because embedding distance can also reflect paraphrasing, we additionally use GPT-5.1~\citep{gpt5} with the 0--3 mathematical-strategy rubric from APMPO~\citep{apmpo}. The judge scores non-overlapping groups of eight correct responses, distinguishing same-logic rephrasings from changes in derivation or mathematical approach. We retain prompts for which both models provide at least one such group. Due to API cost, we apply this analysis to the four benchmarks with the fewest problems; the resulting common sets contain 9, 8, 33, and 147 prompts for AIME24, AIME25, AMC23, and Minerva, respectively.

\begin{table}[t]
\centering
\small
\setlength{\tabcolsep}{2.8pt}
\begin{tabular}{l|cc|cc}
\Xhline{1.1pt}
\rule{0pt}{2.0ex} & \multicolumn{2}{c|}{\textbf{Embedding Distance}} & \multicolumn{2}{c}{\textbf{GPT Strategy Score}} \\
\hhline{~|--|--}
\rule{0pt}{2.0ex}\textbf{Benchmark} & \textbf{GRPO} & \textbf{F-GRPO} & \textbf{GRPO} & \textbf{F-GRPO} \\
\Xhline{1.1pt}
\rule{0pt}{2.0ex}AIME24 & 0.0110 & \textbf{0.0179} & 0.9979 & \textbf{1.4726} \\
AIME25 & 0.0205 & \textbf{0.0223} & 0.9136 & \textbf{1.0518} \\
AMC23 & 0.0105 & \textbf{0.0136} & 0.7742 & \textbf{0.8910} \\
MATH500 & 0.0110 & \textbf{0.0125} & -- & -- \\
Minerva & 0.0216 & \textbf{0.0282} & 0.5896 & \textbf{0.6852} \\
Olympiad & 0.0134 & \textbf{0.0158} & -- & -- \\
\Xhline{0.2pt}
\rule{0pt}{2.0ex}\textbf{All Math} & 0.0147 & \textbf{0.0184} & -- & -- \\
\Xhline{1.1pt}
\end{tabular}
\caption{Correct-solution diversity for the selected Qwen2.5-7B GRPO and F-GRPO checkpoints at $N{=}8$; F-GRPO uses $\gamma{=}0.5$. Higher is better. Embedding values are mean within-prompt pairwise cosine distances; GPT values are mean 0--3 mathematical-strategy scores. All Math is the mean of the six embedding-distance values.}
\label{tab:solution_diversity}
\end{table}

F-GRPO has higher embedding distance on all six benchmarks and a higher GPT strategy score on all four judged benchmarks.

\section{Regularization and Update-Scale Controls: Full Results}
\label{app:grpo_controls}

\paragraph{Learning-rate multiplier.}
For GRPO low-LR, we match the average advantage-magnitude scale induced by Focal weighting. Write $\widehat{\mu}=\widehat{\mu}_{\mathrm{pos}}(x)$. With binary rewards, the mean absolute normalized GRPO advantage at sampled success rate $\widehat{\mu}$ is proportional to $\sqrt{\widehat{\mu}(1-\widehat{\mu})}$, while F-GRPO multiplies it by $(1-\widehat{\mu})^\gamma$. Averaging this multiplicative reduction over the success-rate axis $\widehat{\mu}\in[0,1]$ gives
\[
m_\gamma =
\frac{\int_0^1 (1-\widehat{\mu})^\gamma \sqrt{\widehat{\mu}(1-\widehat{\mu})}\,d\widehat{\mu}}{\int_0^1 \sqrt{\widehat{\mu}(1-\widehat{\mu})}\,d\widehat{\mu}}.
\]
The numerator is $B(\frac{3}{2},\gamma+\frac{3}{2})$ and the denominator is $B(\frac{3}{2},\frac{3}{2})$, where $B(a,b):=\int_0^1 t^{a-1}(1-t)^{b-1}\,dt$ is the beta function. Therefore,
\[
m_\gamma =
\frac{B(\frac{3}{2},\gamma+\frac{3}{2})}{B(\frac{3}{2},\frac{3}{2})}.
\]
Using $B(a,b)=\Gamma(a)\Gamma(b)/\Gamma(a+b)$ together with $\Gamma(\frac{3}{2})=\sqrt{\pi}/2$ and $\Gamma(3)=2$, we obtain
\[
m_\gamma =
\frac{4}{\sqrt{\pi}}\frac{\Gamma(\gamma+\frac{3}{2})}{\Gamma(\gamma+3)}.
\]

\section{LLM Pass@$k$ Tables}
\label{app:llm_passk_curves}

Tables \ref{tab:passk_qwen2_5_7b_pairwise}-\ref{tab:passk_llama_3_2_3b_instruct_pairwise} report benchmark-level pass@$k$ values up to $k{=}256$ for the same selected checkpoints used in Table~\ref{tab:all_models}. For each model, we follow the same evaluation protocol as in the main results. Each cell shows baseline/focal pass@$k$, with the better value in bold.

\begin{table*}[t]
\centering
\setlength{\tabcolsep}{4.5pt}
\resizebox{\textwidth}{!}{%
\begin{tabular}{cr|cccccc|ccc}
\toprule
\textbf{Method} & \textbf{$k$} & \textbf{AIME24} & \textbf{AIME25} & \textbf{AMC23} & \textbf{MATH500} & \textbf{Minerva} & \textbf{Olympiad} & \textbf{IFEval} & \textbf{SynLogic} & \textbf{GPQA} \\
\midrule
\multirow{9}{*}{\pairmethod{GRPO}{F-GRPO}} & 1 & 15.0/\textbf{15.9} & 6.7/\textbf{10.1} & 52.9/\textbf{56.2} & 75.8/\textbf{76.2} & \textbf{36.0}/35.7 & \textbf{37.8}/37.5 & 32.1/\textbf{34.0} & 7.9/\textbf{8.7} & 11.3/\textbf{15.0} \\
 & 2 & 17.4/\textbf{20.1} & 9.4/\textbf{14.2} & 59.9/\textbf{65.9} & 80.0/\textbf{81.2} & \textbf{41.1}/40.9 & 43.5/\textbf{43.9} & 39.9/\textbf{42.7} & 11.0/\textbf{12.1} & 16.8/\textbf{21.8} \\
 & 4 & 19.8/\textbf{23.2} & 13.4/\textbf{19.1} & 65.1/\textbf{73.5} & 83.2/\textbf{84.8} & 45.4/45.4 & 48.6/\textbf{49.6} & 46.5/\textbf{49.9} & 15.1/\textbf{16.8} & 22.5/\textbf{29.1} \\
 & 8 & 22.6/\textbf{26.2} & 18.3/\textbf{23.9} & 69.2/\textbf{79.3} & 85.9/\textbf{87.7} & 49.0/49.0 & 52.9/\textbf{54.4} & 52.0/\textbf{55.9} & 20.2/\textbf{22.9} & 28.0/\textbf{36.0} \\
 & 16 & 25.6/\textbf{30.0} & 23.1/\textbf{28.3} & 73.1/\textbf{83.5} & 88.0/\textbf{90.0} & 51.9/51.9 & 56.6/\textbf{58.4} & 56.7/\textbf{60.9} & 26.2/\textbf{30.1} & 32.7/\textbf{41.8} \\
 & 32 & 29.0/\textbf{34.6} & 27.1/\textbf{32.8} & 77.2/\textbf{87.2} & 89.5/\textbf{91.8} & \textbf{54.5}/54.3 & 59.6/\textbf{62.0} & 60.9/\textbf{65.5} & 32.7/\textbf{37.6} & 36.8/\textbf{46.7} \\
 & 64 & 32.2/\textbf{39.2} & 30.9/\textbf{38.2} & 81.0/\textbf{91.2} & 90.7/\textbf{93.1} & \textbf{56.8}/56.4 & 62.0/\textbf{65.5} & 64.5/\textbf{69.7} & 39.4/\textbf{44.8} & 40.4/\textbf{50.8} \\
 & 128 & 35.0/\textbf{42.9} & 35.2/\textbf{44.8} & 84.5/\textbf{94.5} & 91.8/\textbf{94.2} & \textbf{58.7}/58.4 & 63.9/\textbf{68.8} & 67.7/\textbf{73.1} & 45.7/\textbf{51.3} & 43.5/\textbf{54.2} \\
 & 256 & 37.7/\textbf{46.2} & 40.8/\textbf{52.6} & 87.3/\textbf{96.3} & 92.8/\textbf{95.1} & 60.2/\textbf{60.3} & 65.8/\textbf{71.6} & 70.3/\textbf{75.7} & 51.3/\textbf{57.0} & 46.2/\textbf{57.3} \\
\midrule
\multirow{9}{*}{\pairmethod{DAPO}{F-DAPO}} & 1 & 16.8/\textbf{20.9} & \textbf{12.0}/11.5 & 53.3/\textbf{55.9} & 78.6/\textbf{79.1} & \textbf{35.5}/35.0 & 40.5/\textbf{40.9} & 24.1/\textbf{30.8} & 7.5/\textbf{7.9} & \textbf{15.4}/15.0 \\
 & 2 & 20.6/\textbf{25.3} & 15.6/\textbf{15.7} & 61.6/\textbf{64.2} & 82.9/\textbf{84.4} & 41.3/41.3 & 46.6/\textbf{48.2} & 33.1/\textbf{39.1} & 10.1/\textbf{11.4} & \textbf{22.2}/21.9 \\
 & 4 & 23.9/\textbf{28.9} & 19.8/\textbf{20.1} & 68.5/\textbf{70.9} & 86.1/\textbf{88.0} & 45.8/\textbf{46.0} & 51.6/\textbf{54.3} & 41.0/\textbf{46.1} & 13.8/\textbf{16.6} & 28.8/\textbf{28.9} \\
 & 8 & 27.4/\textbf{32.6} & 24.4/\textbf{25.3} & 74.3/\textbf{77.3} & 88.6/\textbf{90.8} & 49.6/\textbf{50.0} & 56.1/\textbf{59.2} & 47.5/\textbf{51.8} & 19.0/\textbf{23.6} & 34.3/\textbf{34.6} \\
 & 16 & 31.1/\textbf{36.6} & 29.1/\textbf{31.1} & 79.4/\textbf{83.0} & 90.6/\textbf{92.9} & 52.8/\textbf{53.4} & 60.0/\textbf{63.3} & 52.8/\textbf{56.6} & 25.4/\textbf{31.9} & 38.7/\textbf{39.7} \\
 & 32 & 35.2/\textbf{40.7} & 34.2/\textbf{37.2} & 83.9/\textbf{87.8} & 92.1/\textbf{94.4} & 55.5/\textbf{56.1} & 63.3/\textbf{67.0} & 57.3/\textbf{60.9} & 32.4/\textbf{40.3} & 42.8/\textbf{44.7} \\
 & 64 & 39.5/\textbf{44.7} & 39.0/\textbf{43.1} & 87.6/\textbf{91.0} & 93.4/\textbf{95.4} & 57.8/\textbf{58.6} & 66.3/\textbf{70.3} & 61.1/\textbf{64.7} & 39.3/\textbf{48.2} & 46.9/\textbf{49.8} \\
 & 128 & 44.3/\textbf{49.0} & 42.6/\textbf{48.2} & 90.4/\textbf{92.8} & 94.4/\textbf{96.1} & 59.7/\textbf{60.9} & 69.0/\textbf{73.1} & 64.4/\textbf{68.0} & 46.1/\textbf{55.5} & 51.0/\textbf{54.0} \\
 & 256 & 49.8/\textbf{53.4} & 45.6/\textbf{52.9} & 91.9/\textbf{93.7} & 95.2/\textbf{96.6} & 61.2/\textbf{62.9} & 71.8/\textbf{75.6} & 67.1/\textbf{71.1} & 53.3/\textbf{62.4} & 54.9/\textbf{57.4} \\
\midrule
\multirow{9}{*}{\pairmethod{CISPO}{F-CISPO}} & 1 & 14.6/\textbf{14.8} & 9.7/\textbf{13.0} & \textbf{57.8}/53.3 & 78.7/\textbf{79.0} & \textbf{34.7}/34.6 & 41.5/\textbf{42.4} & 24.2/\textbf{30.7} & 8.0/\textbf{8.2} & 12.6/\textbf{15.4} \\
 & 2 & 18.7/\textbf{20.2} & 14.2/\textbf{18.0} & \textbf{67.4}/64.2 & 84.4/\textbf{85.0} & 40.5/\textbf{41.7} & 48.9/\textbf{50.1} & 32.2/\textbf{38.7} & 11.3/\textbf{11.7} & 19.2/\textbf{23.5} \\
 & 4 & 22.9/\textbf{25.3} & 19.6/\textbf{22.8} & 74.9/\textbf{75.0} & 88.4/\textbf{88.9} & 45.8/\textbf{47.4} & 55.3/\textbf{56.3} & 39.1/\textbf{45.3} & 15.8/\textbf{16.7} & 26.6/\textbf{32.1} \\
 & 8 & 27.8/\textbf{30.8} & 26.0/\textbf{28.0} & 82.0/\textbf{83.0} & 91.2/\textbf{91.7} & 50.5/\textbf{51.8} & 60.6/\textbf{61.3} & 45.3/\textbf{51.0} & 21.6/\textbf{23.1} & 33.7/\textbf{40.3} \\
 & 16 & 32.9/\textbf{36.8} & 33.0/\textbf{34.4} & 87.4/\textbf{89.5} & 93.2/\textbf{93.9} & 54.2/\textbf{55.1} & 65.1/\textbf{65.6} & 50.9/\textbf{55.9} & 28.0/\textbf{30.5} & 39.7/\textbf{47.7} \\
 & 32 & 37.3/\textbf{43.4} & 40.3/\textbf{42.1} & 91.5/\textbf{93.2} & 94.7/\textbf{95.5} & 57.0/\textbf{57.7} & 68.8/\textbf{69.2} & 55.9/\textbf{60.1} & 34.5/\textbf{38.2} & 44.8/\textbf{54.2} \\
 & 64 & 40.3/\textbf{50.0} & 47.3/\textbf{50.5} & 94.2/\textbf{95.1} & 95.7/\textbf{96.6} & 59.3/\textbf{60.1} & 71.9/\textbf{72.1} & 60.4/\textbf{64.0} & 40.8/\textbf{45.8} & 49.1/\textbf{59.6} \\
 & 128 & 42.7/\textbf{55.4} & 54.0/\textbf{58.4} & 95.5/\textbf{96.2} & 96.5/\textbf{97.3} & 61.4/\textbf{62.3} & 74.6/\textbf{74.9} & 64.3/\textbf{67.4} & 47.2/\textbf{53.1} & 52.6/\textbf{63.9} \\
 & 256 & 45.9/\textbf{59.7} & 59.8/\textbf{64.6} & 96.1/\textbf{97.1} & 97.0/\textbf{97.8} & 63.3/\textbf{64.3} & 76.9/\textbf{77.5} & 67.9/\textbf{70.6} & 53.6/\textbf{60.0} & 55.5/\textbf{67.1} \\
\bottomrule
\end{tabular}%
}
\caption{Qwen2.5-7B pass@$k$ up to $256$ for the selected checkpoints from Table~\ref{tab:all_models}. Cells report baseline/focal values.}
\label{tab:passk_qwen2_5_7b_pairwise}
\end{table*}

\begin{table*}[t]
\centering
\setlength{\tabcolsep}{4.5pt}
\resizebox{\textwidth}{!}{%
\begin{tabular}{cr|cccccc|ccc}
\toprule
\textbf{Method} & \textbf{$k$} & \textbf{AIME24} & \textbf{AIME25} & \textbf{AMC23} & \textbf{MATH500} & \textbf{Minerva} & \textbf{Olympiad} & \textbf{IFEval} & \textbf{SynLogic} & \textbf{GPQA} \\
\midrule
\multirow{9}{*}{\pairmethod{GRPO}{F-GRPO}} & 1 & 14.8/\textbf{17.4} & 8.3/\textbf{15.8} & 54.6/\textbf{58.3} & 79.5/\textbf{82.3} & \textbf{39.2}/38.5 & 43.1/\textbf{43.7} & 36.9/\textbf{41.2} & 10.7/10.7 & 19.0/\textbf{21.1} \\
 & 2 & 18.4/\textbf{23.0} & 12.5/\textbf{20.9} & 62.7/\textbf{68.4} & 84.3/\textbf{87.0} & 43.7/\textbf{44.2} & 49.3/\textbf{49.9} & 47.8/\textbf{53.1} & 14.6/\textbf{15.3} & 26.0/\textbf{28.3} \\
 & 4 & 22.1/\textbf{29.2} & 17.6/\textbf{25.7} & 70.0/\textbf{77.2} & 87.6/\textbf{90.1} & 47.6/\textbf{48.6} & 54.6/\textbf{55.1} & 57.8/\textbf{63.6} & 19.3/\textbf{21.1} & 32.6/\textbf{35.1} \\
 & 8 & 26.4/\textbf{35.8} & 23.2/\textbf{30.7} & 76.3/\textbf{84.0} & 90.0/\textbf{92.4} & 51.0/\textbf{52.1} & 59.1/\textbf{59.6} & 66.6/\textbf{72.1} & 24.9/\textbf{28.0} & 38.0/\textbf{41.2} \\
 & 16 & 31.8/\textbf{42.1} & 28.5/\textbf{35.8} & 81.8/\textbf{88.3} & 91.9/\textbf{94.0} & 53.9/\textbf{55.2} & 62.8/\textbf{63.6} & 73.8/\textbf{78.5} & 31.3/\textbf{35.8} & 42.3/\textbf{46.6} \\
 & 32 & 37.7/\textbf{47.1} & 33.3/\textbf{40.8} & 86.7/\textbf{91.1} & 93.4/\textbf{95.3} & 56.4/\textbf{58.0} & 66.2/\textbf{67.1} & 79.4/\textbf{83.3} & 38.2/\textbf{43.8} & 45.7/\textbf{51.5} \\
 & 64 & 43.4/\textbf{51.4} & 37.6/\textbf{46.2} & 91.0/\textbf{93.6} & 94.5/\textbf{96.2} & 58.5/\textbf{60.3} & 69.3/\textbf{70.2} & 83.8/\textbf{87.3} & 45.0/\textbf{51.4} & 49.0/\textbf{55.7} \\
 & 128 & 48.4/\textbf{55.9} & 41.4/\textbf{52.7} & 94.3/\textbf{96.0} & 95.5/\textbf{96.8} & 60.1/\textbf{62.1} & 72.0/\textbf{73.0} & 87.7/\textbf{90.8} & 51.3/\textbf{58.3} & 52.3/\textbf{59.5} \\
 & 256 & 52.5/\textbf{61.1} & 45.6/\textbf{60.0} & 96.2/\textbf{98.2} & 96.4/\textbf{97.4} & 61.7/\textbf{63.6} & 74.2/\textbf{75.5} & 91.2/\textbf{93.6} & 57.1/\textbf{64.5} & 55.4/\textbf{62.8} \\
\midrule
\multirow{9}{*}{\pairmethod{DAPO}{F-DAPO}} & 1 & 18.6/\textbf{19.1} & 17.6/\textbf{17.7} & 62.0/\textbf{63.0} & 84.4/84.4 & \textbf{40.0}/39.5 & 47.8/\textbf{48.2} & 33.8/\textbf{39.6} & \textbf{11.8}/11.3 & 16.4/\textbf{18.7} \\
 & 2 & 23.2/\textbf{23.7} & 21.8/\textbf{21.9} & 71.6/\textbf{72.8} & 88.7/\textbf{88.9} & 46.1/\textbf{46.2} & 54.9/\textbf{55.3} & 45.8/\textbf{52.4} & 17.3/\textbf{17.5} & 23.8/\textbf{27.0} \\
 & 4 & 28.0/28.0 & 25.7/\textbf{26.3} & 78.7/\textbf{79.9} & 91.4/91.4 & 51.1/\textbf{51.5} & 60.3/\textbf{60.7} & 56.5/\textbf{63.2} & 24.2/\textbf{25.1} & 31.7/\textbf{35.3} \\
 & 8 & \textbf{33.6}/32.8 & 30.2/\textbf{31.8} & 83.5/\textbf{84.6} & \textbf{93.2}/93.0 & 55.3/\textbf{55.6} & 64.5/\textbf{64.9} & 65.3/\textbf{71.4} & 32.0/\textbf{33.3} & 38.7/\textbf{42.6} \\
 & 16 & \textbf{39.7}/38.3 & 35.6/\textbf{38.0} & 86.7/\textbf{88.3} & \textbf{94.7}/94.3 & 58.3/\textbf{59.0} & 67.8/\textbf{68.3} & 72.4/\textbf{77.3} & 40.1/\textbf{42.0} & 44.7/\textbf{49.0} \\
 & 32 & \textbf{45.9}/43.6 & 41.5/\textbf{44.1} & 89.3/\textbf{91.6} & \textbf{95.9}/95.3 & 60.7/\textbf{61.8} & 70.6/\textbf{71.0} & 77.8/\textbf{81.6} & 48.6/\textbf{50.8} & 50.3/\textbf{54.8} \\
 & 64 & \textbf{52.4}/48.1 & 46.7/\textbf{49.2} & 92.0/\textbf{94.1} & \textbf{96.9}/96.3 & 62.9/\textbf{64.1} & 73.2/\textbf{73.5} & 81.8/\textbf{85.2} & 56.8/\textbf{59.1} & 55.3/\textbf{60.1} \\
 & 128 & \textbf{58.5}/53.1 & 50.4/\textbf{53.5} & 94.5/\textbf{95.6} & \textbf{97.5}/97.1 & 65.1/\textbf{66.0} & 75.6/\textbf{75.8} & 85.2/\textbf{88.1} & 64.3/\textbf{66.2} & 59.9/\textbf{64.4} \\
 & 256 & \textbf{63.8}/59.7 & 54.0/\textbf{58.3} & 96.2/\textbf{96.8} & \textbf{97.8}/97.7 & 67.1/\textbf{67.9} & 77.9/\textbf{78.0} & 88.0/\textbf{90.6} & 71.3/\textbf{72.4} & 63.9/\textbf{67.5} \\
\midrule
\multirow{9}{*}{\pairmethod{CISPO}{F-CISPO}} & 1 & 21.5/21.5 & 17.6/\textbf{18.6} & \textbf{63.1}/60.7 & 84.2/84.2 & 40.2/\textbf{42.1} & 48.4/\textbf{48.9} & 32.6/\textbf{35.9} & \textbf{11.6}/11.1 & 23.6/\textbf{24.6} \\
 & 2 & \textbf{27.5}/27.2 & 20.9/\textbf{22.4} & \textbf{72.0}/71.3 & 88.3/\textbf{88.7} & 46.0/\textbf{48.2} & 55.2/\textbf{55.8} & 42.6/\textbf{47.3} & \textbf{17.1}/16.1 & 32.5/\textbf{34.3} \\
 & 4 & 32.5/\textbf{32.8} & 24.7/\textbf{26.8} & \textbf{80.7}/80.5 & 91.0/\textbf{91.5} & 50.7/\textbf{52.7} & 60.5/\textbf{61.2} & 51.8/\textbf{57.5} & \textbf{24.0}/22.7 & 41.0/\textbf{43.7} \\
 & 8 & 37.4/\textbf{38.5} & 30.0/\textbf{32.5} & 85.9/\textbf{86.0} & 92.9/\textbf{93.6} & 54.4/\textbf{56.2} & 64.5/\textbf{65.3} & 59.9/\textbf{66.4} & \textbf{32.4}/30.9 & 47.9/\textbf{51.7} \\
 & 16 & 42.6/\textbf{44.2} & 36.4/\textbf{39.2} & 89.4/\textbf{89.7} & 94.4/\textbf{95.0} & 57.3/\textbf{59.0} & 67.7/\textbf{68.6} & 67.1/\textbf{73.5} & \textbf{41.4}/39.8 & 53.4/\textbf{58.4} \\
 & 32 & 47.6/\textbf{50.4} & 42.4/\textbf{45.6} & 92.0/\textbf{92.3} & 95.6/\textbf{96.0} & 59.5/\textbf{61.5} & 70.4/\textbf{71.6} & 73.2/\textbf{78.9} & \textbf{50.1}/48.2 & 58.3/\textbf{64.6} \\
 & 64 & 51.8/\textbf{56.6} & 46.8/\textbf{51.2} & 94.0/\textbf{94.7} & 96.6/\textbf{96.9} & 61.5/\textbf{63.7} & 72.7/\textbf{74.3} & 78.2/\textbf{83.2} & \textbf{57.8}/55.7 & 63.2/\textbf{70.8} \\
 & 128 & 55.5/\textbf{63.0} & 50.7/\textbf{56.5} & 96.5/\textbf{96.6} & 97.4/\textbf{97.6} & 63.4/\textbf{65.5} & 74.7/\textbf{76.7} & 82.4/\textbf{87.0} & \textbf{64.8}/62.3 & 67.9/\textbf{76.6} \\
 & 256 & 59.3/\textbf{69.5} & 55.4/\textbf{62.7} & 97.4/\textbf{98.6} & 97.9/\textbf{98.2} & 65.1/\textbf{67.2} & 76.6/\textbf{78.8} & 85.9/\textbf{90.3} & \textbf{71.1}/68.0 & 72.4/\textbf{81.6} \\
\bottomrule
\end{tabular}%
}
\caption{Qwen3-4B-Base pass@$k$ up to $256$ for the selected checkpoints from Table~\ref{tab:all_models}. Cells report baseline/focal values.}
\label{tab:passk_qwen3_4b_base_pairwise}
\end{table*}

\begin{table*}[t]
\centering
\setlength{\tabcolsep}{4.5pt}
\resizebox{\textwidth}{!}{%
\begin{tabular}{cr|cccccc|ccc}
\toprule
\textbf{Method} & \textbf{$k$} & \textbf{AIME24} & \textbf{AIME25} & \textbf{AMC23} & \textbf{MATH500} & \textbf{Minerva} & \textbf{Olympiad} & \textbf{IFEval} & \textbf{SynLogic} & \textbf{GPQA} \\
\midrule
\multirow{9}{*}{\pairmethod{GRPO}{F-GRPO}} & 1 & 10.7/\textbf{12.1} & 0.7/\textbf{1.0} & \textbf{30.5}/29.8 & \textbf{55.0}/54.1 & \textbf{21.8}/21.0 & 19.4/\textbf{20.1} & 54.1/\textbf{56.4} & \textbf{4.7}/4.6 & \textbf{17.5}/15.2 \\
 & 2 & 15.8/\textbf{17.5} & 1.4/\textbf{1.8} & 41.0/\textbf{41.4} & \textbf{63.4}/63.1 & 28.1/28.1 & 25.5/\textbf{26.7} & 61.7/\textbf{64.0} & \textbf{7.0}/6.8 & \textbf{24.4}/22.4 \\
 & 4 & 20.6/\textbf{22.9} & 2.6/\textbf{3.2} & 50.6/\textbf{52.4} & 69.7/\textbf{70.1} & 34.0/\textbf{34.7} & 31.1/\textbf{32.8} & 66.7/\textbf{68.9} & \textbf{9.8}/9.4 & 29.5/\textbf{30.7} \\
 & 8 & 24.3/\textbf{27.5} & 4.6/\textbf{5.1} & 58.7/\textbf{61.9} & 74.9/\textbf{75.7} & 39.2/\textbf{40.3} & 36.4/\textbf{38.4} & 70.0/\textbf{72.2} & \textbf{13.3}/12.7 & 35.9/\textbf{36.3} \\
 & 16 & 26.9/\textbf{30.7} & 7.4/\textbf{7.6} & 65.7/\textbf{70.1} & 79.3/\textbf{80.4} & 43.6/\textbf{44.9} & 41.4/\textbf{43.7} & 72.1/\textbf{74.5} & \textbf{17.7}/17.0 & 41.2/\textbf{41.7} \\
 & 32 & 29.2/\textbf{33.0} & 10.3/\textbf{10.6} & 72.4/\textbf{77.2} & 82.9/\textbf{84.7} & 47.6/\textbf{48.9} & 46.1/\textbf{48.6} & 73.8/\textbf{76.3} & \textbf{22.8}/21.9 & 45.6/\textbf{47.0} \\
 & 64 & 32.0/\textbf{35.8} & 13.3/\textbf{14.3} & 78.7/\textbf{83.1} & 86.0/\textbf{88.3} & 51.4/\textbf{52.8} & 50.7/\textbf{53.2} & 75.3/\textbf{77.7} & \textbf{27.8}/26.9 & 49.6/\textbf{51.5} \\
 & 128 & 35.8/\textbf{40.0} & 16.5/\textbf{20.5} & 83.9/\textbf{87.6} & 88.6/\textbf{91.1} & 55.2/\textbf{56.5} & 55.2/\textbf{57.4} & 76.7/\textbf{78.7} & \textbf{32.3}/31.5 & 52.9/\textbf{55.1} \\
 & 256 & 40.7/\textbf{46.1} & 21.5/\textbf{29.5} & 88.2/\textbf{90.6} & 90.6/\textbf{92.9} & 59.0/\textbf{60.1} & 59.3/\textbf{61.3} & 78.0/\textbf{79.6} & \textbf{36.4}/35.5 & 55.1/\textbf{57.6} \\
\midrule
\multirow{9}{*}{\pairmethod{DAPO}{F-DAPO}} & 1 & \textbf{12.8}/11.1 & 1.0/\textbf{1.7} & \textbf{33.1}/31.9 & 55.9/\textbf{58.6} & \textbf{22.4}/22.3 & 21.0/\textbf{23.2} & 51.2/\textbf{53.0} & \textbf{4.8}/4.3 & 15.7/\textbf{17.0} \\
 & 2 & \textbf{17.2}/16.1 & 1.8/\textbf{3.0} & \textbf{43.3}/42.2 & 61.6/\textbf{66.5} & 27.9/\textbf{29.0} & 25.7/\textbf{29.3} & 59.7/\textbf{61.6} & \textbf{6.7}/6.3 & 21.2/\textbf{23.6} \\
 & 4 & 20.4/\textbf{21.4} & 2.9/\textbf{4.8} & 52.6/\textbf{53.0} & 66.2/\textbf{72.5} & 32.9/\textbf{35.1} & 29.8/\textbf{34.8} & 65.6/\textbf{67.0} & 8.6/\textbf{8.8} & 26.1/\textbf{29.8} \\
 & 8 & 24.8/\textbf{25.0} & 4.4/\textbf{6.8} & 60.4/\textbf{62.0} & 70.0/\textbf{77.5} & 37.3/\textbf{40.1} & 33.4/\textbf{39.7} & 69.2/\textbf{70.5} & 11.0/\textbf{12.2} & 30.6/\textbf{35.5} \\
 & 16 & 28.2/\textbf{28.6} & 6.0/\textbf{8.9} & 65.8/\textbf{69.7} & 73.1/\textbf{81.7} & 41.1/\textbf{44.5} & 36.6/\textbf{44.3} & 71.7/\textbf{73.0} & 14.1/\textbf{16.4} & 34.9/\textbf{40.7} \\
 & 32 & 31.3/\textbf{32.3} & 8.1/\textbf{11.5} & 70.2/\textbf{75.4} & 76.0/\textbf{85.2} & 44.6/\textbf{48.6} & 39.6/\textbf{48.7} & 73.7/\textbf{75.1} & 17.7/\textbf{21.1} & 38.9/\textbf{44.9} \\
 & 64 & 34.5/\textbf{36.2} & 10.8/\textbf{15.5} & 74.5/\textbf{79.9} & 78.9/\textbf{88.1} & 47.8/\textbf{52.5} & 42.5/\textbf{53.1} & 75.4/\textbf{77.0} & 21.6/\textbf{25.6} & 42.2/\textbf{48.4} \\
 & 128 & 37.5/\textbf{40.2} & 14.1/\textbf{21.5} & 77.7/\textbf{84.2} & 81.5/\textbf{90.4} & 50.9/\textbf{56.1} & 45.4/\textbf{57.3} & 76.7/\textbf{78.5} & 25.4/\textbf{29.4} & 44.7/\textbf{51.3} \\
 & 256 & 40.8/\textbf{44.4} & 18.5/\textbf{28.7} & 79.5/\textbf{88.3} & 83.8/\textbf{92.0} & 54.1/\textbf{59.3} & 48.4/\textbf{61.3} & 77.8/\textbf{79.5} & 28.9/\textbf{33.0} & 47.0/\textbf{53.7} \\
\midrule
\multirow{9}{*}{\pairmethod{CISPO}{F-CISPO}} & 1 & 9.7/\textbf{10.6} & 1.0/\textbf{2.0} & 32.9/\textbf{34.1} & \textbf{56.9}/56.5 & 21.8/\textbf{22.1} & \textbf{22.5}/21.5 & \textbf{54.6}/52.6 & 4.3/\textbf{5.4} & \textbf{18.2}/17.0 \\
 & 2 & 15.0/\textbf{15.9} & 1.8/\textbf{3.4} & 43.2/\textbf{44.9} & 64.1/\textbf{64.5} & 28.1/28.1 & \textbf{27.9}/27.5 & \textbf{62.5}/60.6 & 6.2/\textbf{7.4} & \textbf{24.1}/22.8 \\
 & 4 & 21.0/\textbf{21.2} & 3.2/\textbf{5.1} & 53.2/\textbf{53.3} & 69.8/\textbf{70.6} & \textbf{34.0}/33.7 & 32.8/\textbf{33.0} & \textbf{67.6}/65.7 & 8.6/\textbf{9.8} & \textbf{29.5}/27.8 \\
 & 8 & 25.4/\textbf{26.8} & 5.1/\textbf{7.0} & 59.4/\textbf{61.3} & 74.3/\textbf{75.6} & \textbf{39.3}/38.8 & 37.3/\textbf{38.1} & \textbf{70.8}/69.0 & 11.7/\textbf{12.8} & \textbf{34.1}/32.3 \\
 & 16 & 28.7/\textbf{31.3} & 7.5/\textbf{8.9} & 64.1/\textbf{66.9} & 77.9/\textbf{79.8} & \textbf{44.0}/43.3 & 41.4/\textbf{42.8} & \textbf{73.2}/71.3 & 15.1/\textbf{16.6} & \textbf{38.3}/36.2 \\
 & 32 & 31.6/\textbf{34.1} & 10.8/\textbf{11.2} & 68.8/\textbf{70.6} & 81.0/\textbf{83.4} & \textbf{48.4}/47.5 & 45.1/\textbf{47.2} & \textbf{75.0}/73.1 & 18.7/\textbf{20.9} & \textbf{42.2}/39.5 \\
 & 64 & 35.1/\textbf{36.1} & \textbf{15.1}/14.3 & 73.4/\textbf{74.1} & 83.9/\textbf{86.4} & \textbf{52.4}/51.5 & 48.7/\textbf{51.1} & \textbf{76.2}/74.7 & 22.1/\textbf{25.4} & \textbf{45.3}/42.4 \\
 & 128 & 37.9/\textbf{39.0} & \textbf{20.0}/18.6 & 76.1/\textbf{78.9} & 86.7/\textbf{88.9} & \textbf{56.1}/55.3 & 52.1/\textbf{54.8} & \textbf{77.3}/76.2 & 25.7/\textbf{29.7} & \textbf{47.7}/45.1 \\
 & 256 & 39.4/\textbf{42.8} & \textbf{25.4}/24.5 & 79.1/\textbf{82.6} & 89.1/\textbf{91.0} & \textbf{59.5}/58.8 & 55.4/\textbf{58.7} & \textbf{78.4}/77.3 & 29.4/\textbf{33.9} & \textbf{49.7}/47.7 \\
\bottomrule
\end{tabular}%
}
\caption{Llama-3.2-3B-Instruct pass@$k$ up to $256$ for the selected checkpoints from Table~\ref{tab:all_models}. Cells report baseline/focal values.}
\label{tab:passk_llama_3_2_3b_instruct_pairwise}
\end{table*}

\section{LLM Training Dynamics}
\label{app:llm_training_dynamics}

We include training-dynamics plots (Figures \ref{fig:qwen3_grpo_training_dynamics} - \ref{fig:qwen3_cispo_training_dynamics}) for the Qwen3-4B-Base setup, for which denser periodic evaluation was practical because generation length was capped at $4096$ tokens. These curves are intended as diagnostics rather than replacements for the main evaluation protocol. We score AIME25 every $25$ optimization steps using $64$ generations per problem, so the plotted pass@1 and pass@64 are computed from these periodic $64$-sample evaluations rather than the final $1024$-sample evaluation reported in the main tables. In the reward panel, the bold curve shows an exponential moving average (EMA) with smoothing coefficient $0.08$, overlaid on the raw per-step reward trace.

\begin{figure*}[t!]
\centering
\includegraphics[width=\textwidth]{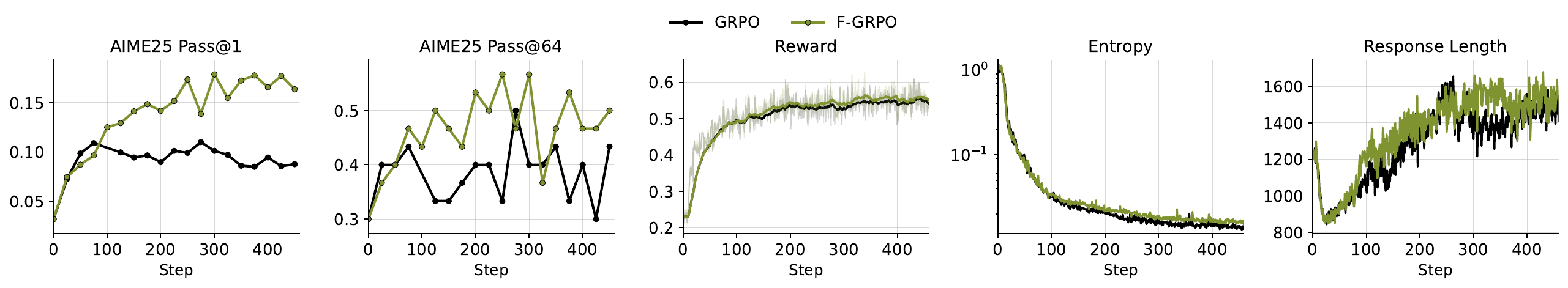}
\caption{Training dynamics on Qwen3-4B-Base for GRPO and F-GRPO (\(\gamma=0.5\)).}
\label{fig:qwen3_grpo_training_dynamics}
\end{figure*}

\begin{figure*}[t!]
\centering
\includegraphics[width=\textwidth]{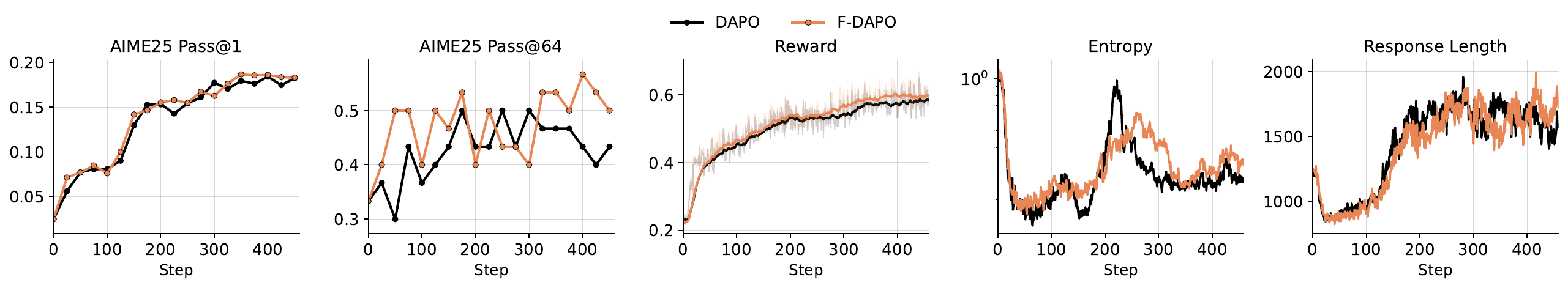}
\caption{Training dynamics on Qwen3-4B-Base for DAPO and F-DAPO (\(\gamma=0.5\)).}
\label{fig:qwen3_dapo_training_dynamics}
\end{figure*}

\begin{figure*}[t!]
\centering
\includegraphics[width=\textwidth]{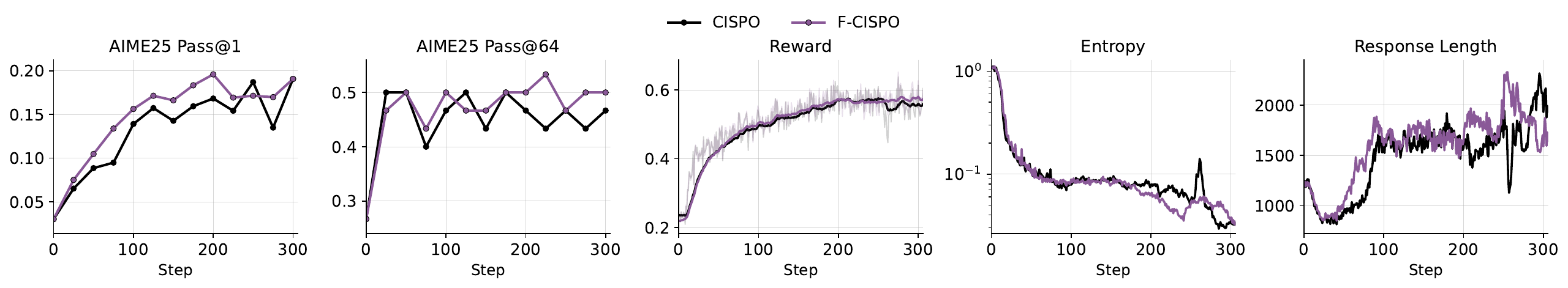}
\caption{Training dynamics on Qwen3-4B-Base for CISPO and F-CISPO (\(\gamma=0.5\)).}
\label{fig:qwen3_cispo_training_dynamics}
\end{figure*}

\section{Statistical Significance}
\label{app:statsign}

To assess the statistical significance of performance differences between the baseline and F-GRPO models, we employ a paired m-out-of-n subsampling test following \citep{subsampling}. For each benchmark, we generate $n=1024$ solutions per problem and use $m=256$ generations (i.e., subsample size $m$) to estimate pass@1 and pass@256 metrics. Specifically, for each subsampling iteration we randomly sample $m=256$ generations without replacement for each problem, compute the pass@$k$ metric using the analytical formula $1 - \binom{n-c}{k} / \binom{n}{k}$ where $n$ is the number of sampled generations and $c$ is the number of correct solutions among them, and average across all problems to obtain a single pass@$k$ estimate for both baseline and F-GRPO models. We perform 50,000 subsampling iterations to obtain the distribution of paired differences in pass@$k$ between the two models.

We conduct a two-sided statistical test with significance level $\alpha = 0.05$. A difference is considered statistically significant if the two-sided $p$-value is less than $0.05$, which is equivalent to the 95\% confidence interval of the subsampling distribution not containing zero.

\section{Notation}
\label{app:notation}

Table~\ref{tab:notation} summarizes the main notation used throughout the paper.

\begin{table*}[t]
\centering
\setlength{\tabcolsep}{0.25em}
\small
\begin{tabular}{lcp{8.7cm}}
\toprule
\multicolumn{1}{c}{\textbf{Category}} & \multicolumn{1}{c}{\textbf{Symbol}} & \multicolumn{1}{c}{\textbf{Meaning}} \\
\midrule
\addlinespace[0.8em]
\multirow{13}{*}{\textbf{Trajectory-level variables}} & $\pi_\theta$ & The policy parameterized by $\theta$ \\
& $D$ & Prompt distribution \\
& $x$ & Given prompt \\
& $o, y$ & A complete response (trajectory) generated by $\pi_\theta$ when given $x$ \\
& $y_t$ & The $t$-th token of response $y$ \\
& $N$ & Group size: number of rollouts sampled per prompt \\
& $R_i$ & Binary reward for rollout $i$ ($R_c$ if correct, $R_w$ if incorrect) \\
& $R_c, R_w$ & Reward values for correct and incorrect rollouts ($R_c > R_w$) \\
& $\mu_{\mathrm{pos}}(x)$ & Success probability: $\Pr_{o\sim\pi_\theta(\cdot|x)}[o\in \mathcal{C}(x)]$ \\
& $X_x$ & Number of correct rollouts for prompt $x$ in a sampled group \\
& $E_x$ & Target subset of correct rollouts for prompt $x$, $E_x\subseteq\mathcal{C}(x)$ \\
& $\tau_E(x)$ & Conditional mass of $E_x$: $\Pr_{o\sim\pi_\theta(\cdot|x)}[o\in E_x]$ \\
& $\widehat{\mu}_{\mathrm{pos}}(x)$ & Empirical success rate: fraction of correct rollouts in the sampled group \\
\addlinespace[0.8em]
\midrule
\addlinespace[0.8em]
\multirow{10}{*}{\textbf{Categorical framework variables}} & $p = \mathrm{softmax}(z)$ & Policy over finite action space $\mathcal{A}$ \\
& $z_i$ & Logit for action $i$ \\
& $\mathcal{P}, \mathcal{I}$ & Sets of correct and incorrect actions \\
& $S_+, S_-, U$ & Sampled correct actions, sampled incorrect actions, and unsampled actions \\
& $Q_{\mathrm{pos}}, Q_{\mathrm{neg}}$ & Total correct and incorrect probability masses \\
& $P_{\mathrm{pos}}, P_{\mathrm{neg}}$ & Sampled correct and incorrect probability masses \\
& $Q_{\mathrm{u,pos}}$ & Unsampled-correct probability mass \\
& $S_{+,2}, S_{-,2}$ & Second moments: $\sum_{i \in S_+} p_i^2$, $\sum_{i \in S_-} p_i^2$ \\
& $U_2$ & Unsampled second moment: $\sum_{i \in U} p_i^2$ \\
& $U_{\mathrm{pos},2}, U_{\mathrm{neg},2}$ & Unsampled second moments for correct and incorrect actions \\
\addlinespace[0.8em]
\midrule
\addlinespace[0.8em]
\multirow{14}{*}{\textbf{Expressions and operators}} & $\pi_\theta(\cdot \mid x, y_{<t})$ & Conditional probability of generating token $\cdot$ given prompt $x$ and previous tokens $y_{<t}$ \\
& $\bar{R}$ & Group mean reward: $\frac{1}{N}\sum_{j=1}^N R_j$ \\
& $\sigma_R$ & Standard deviation of rewards in the group \\
& $\widehat{A}_i^{\mathrm{GRPO}}$ & Group-relative advantage: $(R_i - \bar{R})/(\sigma_R + \epsilon)$ \\
& $\widehat{A}_i^{\mathrm{F-GRPO}}$ & Focal-weighted advantage: $g(x) \cdot \widehat{A}_i^{\mathrm{GRPO}}$ \\
& $r_{i,t}(\theta)$ & Importance ratio: $\pi_\theta(y_{i,t} \mid x, y_{i,<t}) / \pi_{\theta_{\mathrm{old}}}(y_{i,t} \mid x, y_{i,<t})$ \\
& $S_R$ & Batch baseline: $R_c P_{\mathrm{pos}} + R_w P_{\mathrm{neg}}$ \\
& $\Delta z_i$ & One-step logit update: $\frac{\eta}{N} p_i (R_i - S_R)$ \\
& $\Delta Q_{\mathrm{pos}}$ & One-step change in total correct mass \\
& $\Delta Q_{\mathrm{u,pos}}$ & One-step change in unsampled-correct mass \\
& $g(x)$ & Focal weight: $(1 - \widehat{\mu}_{\mathrm{pos}}(x))^\gamma$ \\
& $\gamma$ & Focal-weight exponent controlling reweighting strength \\
& $\eta$ & Learning rate \\
& $\mathcal{M}_{\mathrm{ret}}(t)$ & Retained positive mass: fraction of initial correct probability that has not decreased at step $t$ \\
\addlinespace[0.8em]
\midrule
\addlinespace[0.8em]
\multirow{3}{*}{\textbf{Events and probabilities}} & $\mathcal{A}_N(x)$ & Active event for prompt $x$: $\{0 < X_x < N\}$ \\
& $\mathcal{B}_{E,N}(x)$ & Tail-miss event: active update for prompt $x$ that samples no rollout from $E_x$ \\
& $\Pr(\mathcal{B}_{E,N}(x)\mid x)$ & Pointwise probability of the tail-miss event \\
\addlinespace[0.8em]
\midrule
\addlinespace[0.8em]
\multirow{4}{*}{\textbf{Sets}} & $\Omega_x$ & Space of complete rollouts for prompt $x$ \\
& $\mathcal{C}(x)$ & Subset of correct rollouts for prompt $x$ \\
& $\mathcal{A}$ & Finite action space in the categorical framework \\
& $\mathcal{A}^+$ & Subset of correct actions in the categorical simulation \\
\bottomrule
\end{tabular}
\vspace{0.5em}
\caption{Notation used in the paper.}
\label{tab:notation}
\end{table*}

\end{document}